%% file: neurips_2026_arxiv.tex
\documentclass{article}

\PassOptionsToPackage{numbers, compress}{natbib}

\usepackage[preprint]{neurips_2026}

\usepackage{wrapfig}

\usepackage[utf8]{inputenc}
\usepackage[T1]{fontenc}
\usepackage{hyperref}
\usepackage{url}
\usepackage{booktabs}
\usepackage{amsfonts}
\usepackage{nicefrac}
\usepackage{microtype}
\usepackage{xcolor}

\usepackage{hyperref}
\usepackage{url}

\usepackage{graphicx}
\usepackage{caption}
\usepackage{amsmath}
\usepackage{svg}
\usepackage{amssymb}
\usepackage{amsthm}
\usepackage{paralist}
\usepackage{booktabs}
\usepackage{tikz}
\usepackage{subcaption}
\usepackage{placeins}

\newcommand{\vertices}{V}
\newcommand{\edges}{E}
\newcommand{\temporalEdges}{\edges^\tau}
\newcommand{\graph}{G}
\newcommand{\temporalGraph}{\graph^\tau}
\newcommand{\staticEdges}{\edges^\text{s}}
\newcommand{\aggrGraph}{\graph^\text{a}}
\newcommand{\concGraph}{\graph^\text{c}}
\newcommand{\eventGraph}{G^\eventEdges}
\newcommand{\eventEdges}{\mathcal{E}}
\newcommand{\augmentedGraph}{\graph^\text{aug}}
\newcommand{\augmentedVertices}{\vertices^\text{aug}}
\newcommand{\augmentedEdges}{\edges^\text{aug}}
\newcommand{\inEdges}{\edges^\text{in}}
\newcommand{\outEdges}{\edges^\text{out}}
\newcommand{\vertex}{v}
\newcommand{\vertexA}{u}
\newcommand{\vertexB}{v}

\newcommand{\edgeA}{e}
\newcommand{\neighbors}{N}
\newcommand{\inNeighbors}{\neighbors_I}
\newcommand{\outNeighbors}{\neighbors_O}
\newcommand{\mapping}{\pi}
\newcommand{\vertexMapping}{\pi_\vertices}
\newcommand{\edgeMapping}{\pi_\edges}
\newcommand{\paths}{P}
\newcommand{\aPath}{p}

\newcommand{\aLabel}{\ell}
\newcommand{\nodeLabel}{\aLabel_\vertices}
\newcommand{\edgeLabel}{\aLabel_\edges}
\newcommand{\nodeLabels}{\mathcal{L}_\vertices}
\newcommand{\edgeLabels}{\mathcal{L}_\edges}
\newcommand{\aggrLabel}{\aLabel^\text{a}}
\newcommand{\concLabel}{\aLabel^\text{c}}
\newcommand{\multiset}[1]{\{\!\!\{#1\}\!\!\}}
\DeclareMathOperator{\enc}{enc}
\DeclareMathOperator{\hash}{hash}
\newcommand{\minTime}{t_\text{min}}
\newcommand{\timestamps}{T}
\newcommand{\aggrIn}[1]{\overrightarrow{f}^{(#1)}_\text{agg}}
\newcommand{\aggrOut}[1]{\overleftarrow{f}^{(#1)}_\text{agg}}

\newcommand{\com}[1]{f^{(#1)}_\text{com}}
\newcommand{\temporalPaths}{\paths^\tau}

\tikzstyle{node}=[circle,draw=black, inner sep=5pt]
\newcommand{\picWidth}{0.8}
\newcommand{\picHeight}{0.8}

\usepackage{hyperref}
\usepackage{cleveref}
\usepackage{url}

\newtheorem{definition}{Definition}

\newtheorem{theorem}{Theorem}
\newtheorem*{theorem*}{Theorem}
\newtheorem{observation}{Observation}

\title{Weisfeiler and Leman Follow the Arrow of Time: Expressive Power of Message Passing in Temporal Event Graphs}

\author{
  Franziska Heeg \\
  Chair of Machine Learning for Complex Networks \\
  Center for AI and Data Science (CAIDAS) \\
  University of Würzburg \\
  Würzburg, Germany \\
  \texttt{franziska.heeg@uni-wuerzburg.de} \\
  \And
  Jonas Sauer \\
  Karlsruhe Institute of Technology \\
  University of Karlsruhe \\
  Karlsruhe, Germany \\
  \texttt{jonas.sauer@kit.edu} \\
  \And
  Petra Mutzel \\
  Chair of Computational Analytics \\
  Institute for Computer Science 1 \\
  University of Bonn \\
  Bonn, Germany \\
  \texttt{pmutzel@uni-bonn.de} \\
  \And
  Ingo Scholtes \\
  Chair of Machine Learning for Complex Networks \\
  Center for AI and Data Science (CAIDAS) \\
  University of Würzburg \\
  Würzburg, Germany \\
  \texttt{ingo.scholtes@uni-wuerzburg.de} \\
}

\begin{document}

\maketitle

\begin{abstract}
An important characteristic of temporal graphs is how the directed arrow of time influences their \emph{causal topology}, i.e., which nodes can possibly influence each other causally via time-respecting paths.
The resulting patterns are often neglected by temporal graph neural networks (TGNNs).
To formally analyze the expressive power of TGNNs, we lack a generalization of graph isomorphism to temporal graphs that captures their causal topology.
Addressing this gap, we introduce \emph{consistent event graph isomorphism}, which utilizes a time-unfolded representation of time-respecting paths in temporal graphs.
We compare this definition with existing notions of temporal graph isomorphisms.
We highlight the advantages of our approach and develop a temporal generalization of the Weisfeiler-Leman algorithm as a one-sided test to distinguish non-isomorphic temporal graphs.
Building on this theoretical foundation, we derive a novel message passing scheme for TGNNs that operates on the event graph representation of temporal graphs and experimentally evaluate it in a temporal graph classification setting.
\end{abstract}

\section{Motivation}
\label{sec:motivation}

Graph neural networks (GNNs) are a cornerstone of deep learning in relational data.
They recently were generalized to temporal GNNs (TGNNs) that capture patterns in time series data on \emph{temporal graphs}, with time-stamped edges.

Such temporal graphs can be categorized into two types: In discrete-time temporal graphs (DTTGs), edges carry coarse-grained timestamps and can naturally be represented as sequences of static snapshots, lending themselves to generalizations of static graph learning.
In contrast, in \emph{continuous-time temporal graphs} edges carry high-resolution, possibly unique timestamps.
Hence snapshot graphs are very sparse, which requires (i) a coarse-graining of time that destroys temporal information, or (ii) learning techniques able to fully utilize it.

To address end-to-end learning tasks in temporal graphs, several TGNN architectures were proposed, which capture different patterns.
Examples include the evolution of node embeddings in consecutive snapshots for discrete-time temporal graphs, or temporal edge activation patterns in continuous-time temporal graphs \citep{longa2023graphneuralnetworkstemporal}.
An important additional characteristic of temporal graphs is how the directed \emph{arrow of time} influences which nodes can possibly \emph{causally} influence each other via time-respecting paths.
As example, consider a temporal graph with two edges connecting \emph{Alice} to \emph{Bob} at timestamp $t$ and \emph{Bob} to \emph{Carol} at timestamp $t'$.
If~$t<t'$, Alice can possibly causally influence Carol via Bob.
Conversely, if~$t'<t$, a causal influence from Alice to Carol is impossible because it would propagate backwards in time.
To prevent wrong interpretations of the term \emph{causal} in our work, we stress that the correct temporal order of edges is a \emph{necessary} but not \emph{sufficient} condition for causal influence.

Hence, considering the arrow of time (and thus which temporal ordering of events cannot mediate causal influence) is an important precondition that enables causality-aware learning in temporal graphs.

Numerous works in network science studied how the temporal ordering of edges in temporal graphs influences connectivity, dynamical processes like spreading or diffusion, node centralities, cluster patterns, or controllability \citep{lentz2013unfolding,rosvall2014memory,Scholtes2014_natcomm,Scholtes2017,badie2022directed,Pfitzner2013_prl}.
These patterns are often neglected by TGNNs, which can limit their performance in high-resolution time series data on temporal graphs.

To formally analyze this aspect in existing TGNNs, in line with works on the expressivity of (static) GNNs \citep{xu2019powerfulgraphneuralnetworks,Morris2019}, we lack a generalization of graph isomorphism to temporal graphs that captures how the arrow of time shapes their ``causal topology'' \cite{Scholtes2014_natcomm}.
This could inform the development of \emph{causality-preserving} message passing schemes for TGNNs with provable expressive power.
Addressing this gap, our contributions are:

{\setdefaultleftmargin{1.5em}{}{}{}{}{}
\begin{compactitem}
    \item We propose a new temporal generalization of graph isomorphism called \emph{time-respecting path isomorphism}, which focuses on the preservation of time-respecting paths between temporal graphs.
    While we can also apply it to snapshot-based temporal graphs, our definition is particularly suitable for temporal graphs with high-resolution, possibly unique, timestamps.
    \item Our definition preserves temporal reachability: if there is a time-respecting
    path from $u$ to $v$ in $G_1$, there must also be one between the corresponding
    nodes in any isomorphic $G_2$. Time-respecting paths impose a partial ordering
    on timestamped edges---if a path passes $e$ before $e'$, then $e$ must precede
    $e'$; otherwise their relative order is irrelevant. Two temporal graphs can
    thus be isomorphic even if the total orderings of their timestamped edges
    differ, as long as this partial ordering is preserved.

    \item We contrast our definition with the recently proposed \emph{timewise isomorphism}, which generalizes graph isomorphism to snapshots of discrete-time temporal graphs~\citep{walkega2024expressive}.

     We show that timewise isomorphism is stricter because it also considers the timestamp values. As a consequence, two temporal graphs with identical time-respecting paths can be timewise non-isomorphic.
    \item We show that time-respecting path isomorphism is equivalent to static graph isomorphism on the \emph{augmented event graph}, a newly proposed auxiliary graph that (i) captures time-respecting paths in the temporal graph through a static line graph expansion, and (ii) is augmented by the original nodes in the temporal graph. This allows us to generalize the Weisfeiler-Leman (WL) algorithm, viewed as a one-sided test for distinguishing non-isomorphic graphs, to the temporal setting.
    \item We use our insights to derive a novel message passing scheme operating on the augmented event graph, which generates representations that allow to distinguish non-isomorphic temporal graphs. We show that it has the same expressive power as the WL test on the augmented event graph. We experimentally evaluate the resulting TGNN architecture in a temporal graph classification task with synthetic and real datasets. Compared to TGNNs building on timewise isomorphism, we observe that our model is better able distinguish between differences in temporal activation patterns that affect time-respecting paths, and those that do not.
\end{compactitem}}

Our work contributes to the theoretical foundation of temporal graph learning, providing a basis to develop causality-aware message passing architectures for temporal graphs.

\section{Related Work}
\label{sec:related}

Over the past years, numerous works introduced various generalizations of GNNs for temporal data.
Following the taxonomy given by \citet{longa2023graphneuralnetworkstemporal}, these works can be broadly categorized into snapshot- and event-based models.
\emph{Snapshot-based models} operate on data with low temporal resolution that provide a sequence of static graphs.
In contrast, \emph{event-based models} operate on high-resolution time series data that capture individual events like the addition or removal of nodes or edges.
Here we briefly summarize architectures for these different approaches.

Snapshot-based models such as \texttt{ROLAND} ~\citep{you2022roland} and
EvolveGCN ~\citep{pareja2020evolvegcn} apply RNN-based approaches to sequences of static graph snapshots.
Taking an event-based perspective, the temporal graph network (TGN) architecture~\citep{rossi2020temporalgraphnetworksdeep} integrates  embedding and memory modules to capture multi-faceted patterns in sequences of timestamped edges.
TGAT ~\citep{xu2020inductive} extends the GAT attention mechanism~\citep{velivckovic2017graph} to obtain temporal  encodings based on the changing neighborhood of nodes in temporal graphs.
\citet{wang2021inductive} introduces Causal Anonymous Walks (CAWs), which are temporal random walks anonymized by node hitting counts, enabling inductive representation learning of temporal networks by capturing causal motifs without relying on node identities or rich edge attributes.

None of these methods explicitly model how the \emph{arrow of time} influences time-respecting paths.
A number of works in network science investigated this aspect in temporal graphs \citep{holme2015modern}.
Several works~\citep{lentz2013unfolding,Pfitzner2013_prl,OettershagenKMM20,OettershagenM22,badie2022directed} consider how correlations in the temporal ordering of edges influence connected components, epidemic spreading, and percolation in temporal graphs.
~\citet{rosvall2014memory} study how the temporal ordering of nodes along time-respecting flows influence spreading processes, node centralities and cluster patterns.
~\citet{Scholtes2014_natcomm} use a higher-order Laplacian to analytically predict how the temporal ordering of edges speeds up or slows down diffusion.
Building on this idea, ~\citet{Scholtes2017} use higher-order De Bruijn graphs to model time-respecting paths in temporal graphs.
\citet{qarkaxhija2022bruijn} generalize neural message passing to higher-order De Bruijn graphs, obtaining a TGNN that models how the arrow of time influences time-respecting paths.
~\citet{OettershagenKMM20} suggest a transformation into a static line-graph for which they propose a graph kernel based on the Weisfeiler-Leman algorithm to classify dissemination processes in networks.
In~\Cref{sec:isomorphism} we study the connections to our work.

In recent years, there has been growing interest in theoretically understanding the expressive power of GNNs.
Broadly speaking, the expressivity of a GNN refers to its capacity to capture complex structures and distinguish between different graphs.
In this paper, our notion of expressive power is the ability of a GNN to produce distinct representations for non-isomorphic graphs.
A key insight from recent theoretical works is that the expressive power of message-passing GNNs is fundamentally limited by the 1-dimensional Weisfeiler-Leman (1-WL) graph isomorphism test~\citep{xu2019powerfulgraphneuralnetworks,Morris2019}.
Both works independently showed that there exist GNNs that are at least as powerful as the 1-WL test, e.g., the Graph Isomorphism Network by ~\citet{xu2019powerfulgraphneuralnetworks}.

This led to research pushing GNN expressivity beyond the 1-WL via the $k$-dimensional WL test~\citep{morris2020weisfeiler}, edge-direction-aware variants~\citep{Rossi23}, or structual augmentations (e.g., \citet{BouritsasFZB23}). While these tricks can increase expressivity, they can also introduce challenges such as overfitting or reliance on problem-specific features. In fact,  higher expressivity does not imply better generalization \citet{franks2024weisfeilerlemanmarginexpressivitymatters}.

For a detailed review on WL-based approaches see, e.g., \citet{Morris24a,Morris2023}.

Compared to the static setting, the expressive power of temporal GNNs has not been explored as thoroughly.

In fact, recent works have highlighted that (i) the question which (temporal) patterns state-of-the-art TGNNs can actually learn is poorly understood, and (ii) that the performance of some architectures surprisingly does not seem to be influenced by the time dimension at all \cite{hayes2025temporalgraphlearningmodels}.

This is in part because the time dimension introduces an additional degree of freedom, leaving no universally agreed-upon definition of temporal graph isomorphism.
\citet{BeddarWiesing2024} propose a notion of isomorphism for dynamic graphs, which can be seen as snapshot-based temporal graphs, that considers all snapshots independently of each other.

\citet{walkega2024expressive} observe that this does not fully capture the expressive power of two important classes of TGNN architectures: global and local.
Instead, they propose \emph{timewise isomorphism}, which enforces consistency between the snapshots.
They show that global and local TGNNs differ in their abilities to detect timewise isomorphism, and neither is strictly more powerful than the other.
Similarly, ~\citet{GaoRibeiro22} interpret a temporal graph as a static multi-relational graph in which the timestamps are edge attributes.
Hence, an isomorphism must preserve their exact values.

~\citet{souza2022provably} use this notion of isomorphism to propose a temporal generalization of the 1-WL test.
Furthermore, they study a general class of TGNN architectures, which subsumes both TGN and TGAT, and show that this class has the same expressive power as their proposed temporal 1-WL test when using injective combination and aggregation functions.
For the task of link prediction, they observe that there are events that TGNNs cannot distinguish but CAW can, and vice versa.
To address this, they introduce PINT, an architecture that combines injective temporal message passing with relative positional features inspired by CAW.
Because PINT can distinguish all events that can be distinguished by injective TGNNs or CAW, it is strictly more expressive than both for the task of link prediction.

To our knowledge, no existing notions of temporal graph isomorphism precisely capture the influence of time-respecting reachability, which is crucial to understand the evolution of dynamical processes in temporal graphs.
The goal of an expressivity framework is to characterize which temporal graphs are \emph{fundamentally equivalent}, i.e., indistinguishable regardless of the learning algorithm or timestamp encoding. Timewise isomorphism makes distinguishability contingent on timestamp values: two temporal graphs with identical time-respecting paths but different timestamp values are declared non-isomorphic, so a model that cannot distinguish them is deemed insufficiently expressive---even if no downstream task depends on the specific timestamp values. This is analogous to declaring two static graphs non-isomorphic solely because their adjacency matrices use different node orderings. A principled expressivity analysis should therefore be grounded in an isomorphism notion that reflects structural equivalence rather than representational artifacts. While TGNNs can treat timestamps as features and in principle \emph{learn} invariances to timestamp transformations, our goal is a notion of equivalence that holds \emph{by construction}, independently of what any particular architecture learns.

\section{Preliminaries}
\label{sec:preliminaries}

A directed, labeled (static) graph~$\graph=(\vertices,\edges,\nodeLabel,\edgeLabel)$ consists of a set~$\vertices$ of nodes, a set~$\edges\subseteq \vertices\times \vertices$ of directed edges, a \emph{node labeling} $\nodeLabel\colon\vertices\to\nodeLabels$ and an \emph{edge labeling}~$\edgeLabel\colon\edges\to\edgeLabels$, with countable sets~$\nodeLabels$ and~$\edgeLabels$.
In unlabeled graphs, we omit~$\nodeLabel$ or~$\edgeLabel$ accordingly.
For a node~$\vertexB$, we denote its incoming neighbors by~$\inNeighbors(\vertexB) = \{ \vertexA \mid (\vertexA,\vertexB) \in \edges \}$ and its outgoing neighbors by~$\outNeighbors(\vertexB) = \{ \vertexA \mid (\vertexB,\vertexA) \in \edges \}$.

Finally, we define the set of paths $P(G)$ as the set of all alternating node/edge sequences $(v_0, e_1, v_1, \ldots, e_{k}, v_k)$ with $e_i=(v_{i-1}, v_i) \in E$ for $i \in \{1, \ldots, k\}$.

We do not distinguish between walks and paths or, equivalently, do not require paths to be simple.

\begin{definition}[Graph isomorphism]
\label{def:isomorphism}
For two static graphs~$\graph_1=(\vertices_1,\edges_1,\nodeLabel^1,\edgeLabel^1)$ and~$\graph_2=(\vertices_2,\edges_2,\nodeLabel^2,\edgeLabel^2)$, an \emph{isomorphism} is a bijective mapping~$\mapping\colon\vertices_1\to\vertices_2$ with these properties:
\begin{compactenum}[(i)]
\item \makebox[3.4cm][l]{Edge-preserving:} \makebox[5.7cm][l]{$(\vertexA,\vertexB)\in\edges_1 \ \iff\ (\mapping(\vertexA),\mapping(\vertexB))\in\edges_2$} $\forall \vertexA,\vertexB \in \vertices_1$
\item \makebox[3.4cm][l]{Node label-preserving:} \makebox[5.7cm][l]{$\nodeLabel^1(\vertexA) = \nodeLabel^2(\mapping(\vertexA))$} $\forall \vertexA \in \vertices_1$
\item \makebox[3.4cm][l]{Edge label-preserving:} \makebox[5.7cm][l]{$\edgeLabel^1(\vertexA,\vertexB) = \edgeLabel^2(\mapping(\vertexA),\mapping(\vertexB))$} $\forall (\vertexA,\vertexB) \in \edges_1$
\end{compactenum}
We say that the graphs $G_1$ and $G_2$ are \emph{isomorphic} iff such a mapping $\pi$ exists.
\end{definition}

\begin{definition}[Temporal graph]
	We define a (directed) temporal graph as $\temporalGraph=(\vertices,\temporalEdges)$, where $\vertices$ is the set of nodes and $\temporalEdges\subseteq \vertices\times \vertices\times \mathbb{N}$ is the set of timestamped edges, i.e., an edge $(\vertexA,\vertexB;t)\in\temporalEdges$ means $\vertexA$ and $\vertexB$ interact at time $t$.
\end{definition}

Note that timestamped edges represent instantaneous events, i.e., $(\vertexA,\vertexB;t)\in\temporalEdges$ does not imply $(\vertexA,\vertexB;t')\in\temporalEdges$ for $t \neq t'$.
Like \citet{OettershagenKMM20}, we assume a unit edge traversal time.
Following~\citet{Pan2011}, we assume a maximum time difference $\delta$ between consecutive edges in the following definition of time-respecting paths.
This is crucial as we often consider temporal graphs where the observation period is much longer than the timescale of processes of interest.
As an example, in a social network with timestamped interactions observed over multiple years, information typically propagates within hours or days, i.e., we are not interested in paths where consecutive edges occur in different years.

\begin{definition}[Time-respecting path]
    A path of length~$k$ in a temporal graph~$\temporalGraph=(\vertices,\temporalEdges)$ is an alternating sequence of nodes and timestamped edges~$\aPath=(\vertex_0,\edgeA_1,\vertex_1,\dots,\edgeA_k,\vertex_k)$ with~$\edgeA_i=(\vertex_{i-1},\vertex_i;t_i)\in\temporalEdges$ for~$i\in\{1,\dots,k\}$.
    For a maximum time difference (or waiting time) $\delta \in \mathbb{N}$, we say that~$\aPath$ is \emph{time-respecting} if~$1 \leq t_i - t_{i-1} \leq \delta$ for~$i\in\{2,\dots,k\}$.
    We denote the set of time-respecting paths in~$\temporalGraph$ as~$\temporalPaths(\temporalGraph)$.
\end{definition}

The structure of time-respecting paths can be encoded in the temporal event graph, which is a static graph whose nodes are the timestamped edges. Two nodes are connected by an edge if the corresponding timestamped edges form a time-respecting path of length two.

\begin{definition}[Temporal event graph]
    Let $\temporalGraph = (\vertices,\temporalEdges)$ be a temporal graph with waiting time $\delta$. The temporal event graph is given by $\eventGraph = (\temporalEdges, \eventEdges)$
    with
    \begin{align*}
        \eventEdges = \{((u,v;t),(v,w;t')) \mid &\ (u,v;t),(v,w;t')\in\temporalEdges,\\
        &\ 1 \leq t'-t \leq \delta\}.
    \end{align*}
\end{definition}

Time-respecting paths of length $k \geq 2$ in~$\temporalGraph$ correspond to paths of length~$k-1$ in~$\eventGraph$, whereas time-respecting paths of length~$1$ in~$\temporalGraph$ correspond to nodes in~$\eventGraph$.

Furthermore, we consider two static representations of temporal graphs.

\begin{definition}[Time-aggregated/concatenated static graph]
    \label{def:timeagg}
    For a temporal graph~$\temporalGraph=(\vertices,\temporalEdges)$, let~$\minTime(\temporalGraph) = \min \{ t \mid (\vertexA,\vertexB;t) \in \temporalEdges\}$ denote the earliest timestamp that occurs in~$\temporalGraph$.
    For each pair of nodes~$\vertexA,\vertexB\in\vertices$, let~$\timestamps(\vertexA,\vertexB) = \{ t - \minTime(\temporalGraph) \mid (\vertexA,\vertexB;t) \in \temporalEdges \}$ denote the set of timestamps at which the edge~$(\vertexA,\vertexB)$ occurs, relative to the earliest timestamp.
    Then the set of \emph{static edges} is given by $\staticEdges = \{(\vertexA,\vertexB) \mid \timestamps(\vertexA,\vertexB) \neq \emptyset \}$.
    The \emph{time-aggregated static graph} of~$\temporalGraph$ is the directed, edge-labeled graph $\aggrGraph = (\vertices,\staticEdges,\aggrLabel)$ with edge labels {$\aggrLabel(\vertexA,\vertexB)=|\timestamps(\vertexA,\vertexB)|$}.
    The \emph{time-concatenated static graph} of~$\temporalGraph$ is the directed, edge-labeled graph $\concGraph = (\vertices,\staticEdges,\concLabel)$ with edge labels {$\concLabel(\vertexA,\vertexB)=\timestamps(\vertexA,\vertexB)$}.
\end{definition}

The time-aggregated graph Ga is lossy: it preserves topology and edge frequencies but discards timing. Existence of a time-respecting path in~$\temporalGraph$ implies a corresponding path in~$\aggrGraph$, but the converse fails. The time-concatenated graph $G^c$ is lossless, as each static edge is labeled with all its timestamps.

\section{Isomorphisms in Temporal Graphs}
\label{sec:isomorphism}
To motivate our temporal generalization of graph isomorphism, we make the following observation.

\begin{observation}\label{obs:path-isomorphism}
    Let~$\mapping\colon\vertices_1\to\vertices_2$ be a bijective node mapping between two graphs~$\graph_1=(\vertices_1,\edges_1)$ and~$\graph_2=(\vertices_2,\edges_2)$.
    For any edge~$\edgeA=(\vertexA,\vertexB)\in\edges$, we write~$\mapping(\edgeA)=(\mapping(\vertexA),\mapping(\vertexB))$.
    Then~$\mapping$ is edge-preserving iff it is \emph{path-preserving}, i.e., the following holds for alternating node/edge sequences~$(\vertex_0,\edgeA_1,\vertex_1,\dots,\edgeA_{k},\vertex_k)$ ($k \in \mathbb{N}$):
    \begin{align*}
    &(\vertex_0,\edgeA_1,\vertex_1,\dots,\edgeA_{k},\vertex_k) \in \paths(\graph_1)\\ \Longleftrightarrow\quad&(\pi(\vertex_0),\pi(\edgeA_1),\pi(\vertex_1),\dots,\pi(\edgeA_{k}),\pi(\vertex_k)) \in \paths(\graph_2).
    \end{align*}
\end{observation}

This is due to the fact that adjacent edges transitively expand into paths.
Thus, two isomorphic static graphs are topologically equivalent in terms of edges \emph{and} paths.
Importantly, this property does not directly translate to time-respecting paths in temporal graphs: two adjacent timestamped edges $(u,v;t)$ and $(v,w; t')$ only form a time-respecting path if~$1 \leq t' - t \leq \delta$.
Hence, a temporal generalization of graph isomorphism should preserve not only the timestamped edges, but also the \emph{causal topology} in terms of time-respecting paths.
Conversely, we are interested in an isomorphism definition that does not force \emph{values} of timestamps to be preserved as long as the resulting time-respecting paths are identical.

\begin{definition}[Time-respecting path isomorphism]\label{def:time-resp-path-iso}
    Let $\temporalGraph_1=(\vertices_1,\temporalEdges_1)$ and $\temporalGraph_2=(\vertices_2,\temporalEdges_2)$ be two temporal graphs. We say that~$\temporalGraph_1$ and~$\temporalGraph_2$ are \emph{time-respecting path isomorphic} if there is a bijective node mapping~$\vertexMapping\colon\vertices_1\rightarrow \vertices_2$ and a bijective timestamped edge mapping~$\edgeMapping\colon\temporalEdges_1\rightarrow \temporalEdges_2$ such that the following holds for alternating node/edge sequences~$(\vertex_0,\edgeA_1,\vertex_1,\dots,\edgeA_{k},\vertex_k)$ with~$k \in \mathbb{N}$:
    \begin{align*}
    &(\vertex_0,\edgeA_1,\vertex_1,\dots,\edgeA_{k},\vertex_k) \in \temporalPaths(\temporalGraph_1)\\ \Longleftrightarrow\quad&(\vertexMapping(\vertex_0),\edgeMapping(\edgeA_1),\vertexMapping(\vertex_1),\dots,\edgeMapping(\edgeA_{k}),\vertexMapping(\vertex_k))
    \in \temporalPaths(\temporalGraph_2).
    \end{align*}
\end{definition}
Note that in contrast to \Cref{obs:path-isomorphism}, this definition also includes an edge mapping.
In a static graph, each edge is uniquely defined by a pair of endpoints, so the edge mapping is induced by the node mapping.
This is not the case in temporal graphs, in which multiple timestamped edges may connect the same node pair at different times.
Hence, the edge mapping is specified separately.

Since the number of time-respecting paths may be exponential in the graph size, we derive equivalent notions that are easier to test.
In order to preserve paths of length~$1$, which consist of a single timestamped edge~$\edgeA=(\vertexA,\vertexB;t)$ and are always time-respecting, we must ensure that~$\edgeMapping(\edgeA)$ connects~$\vertexMapping(\vertexA)$ to~$\vertexMapping(\vertexB)$.
We call this property \emph{node consistency}.
Node-consistent mappings preserve paths, but not necessarily their time-respecting property.
To ensure this, we observe that time-respecting paths of length~$k \geq 2$ correspond to paths in the temporal event graph.
We can preserve them by requiring~$\edgeMapping$ to be path-preserving between the temporal event graphs.

\begin{definition}[Consistent event graph isomorphism]
    \label{def:consistent}
    Let $\temporalGraph_1=(\vertices_1,\temporalEdges_1)$ and $\temporalGraph_2=(\vertices_2,\temporalEdges_2)$ be temporal graphs with corresponding temporal event graphs $\eventGraph_1=(\temporalEdges_1, \eventEdges_1)$ and $\eventGraph_2=(\temporalEdges_2, \eventEdges_2)$.
    A mapping~$\edgeMapping\colon\temporalEdges_1\to\temporalEdges_2$ is a \emph{consistent event graph isomorphism} iff the following holds:
    \begin{compactenum}[(i)]
        \item there exists a bijective mapping~$\vertexMapping\colon\vertices_1\to\vertices_2$ such that
        $$\forall (\vertexA,\vertexB;t) \in \temporalEdges_1 \, \exists t'\colon  \edgeMapping(\vertexA,\vertexB;t)=(\vertexMapping(\vertexA), \vertexMapping(\vertexB); t')$$
        \item $\edgeMapping$ is a graph isomorphism between~$\eventGraph_1$ and~$\eventGraph_2$.
    \end{compactenum}
\end{definition}

This can be simplified further by constructing an \emph{augmented event graph} (see~\Cref{fig:aug-event-graph}) that encodes the node consistency property in its topology.
With this, we reduce the test for time-respecting path isomorphism to the test for static graph isomorphism on augmented event graphs.

\begin{figure}[ht!]
    \centering
    \begin{subfigure}[b]{0.2\textwidth}
    \centering
    \input{fig/example_double_small}

    \end{subfigure}
    \hspace{2cm}
    \begin{subfigure}[b]{0.2\textwidth}
    \centering
    \input{fig/augmented}

    \end{subfigure}

    \caption{A temporal graph~$\temporalGraph$ (left) and the corresponding augmented event graph $\augmentedGraph$ (right).
    In~$\augmentedGraph$, gray nodes have label~$0$ and white nodes have label~$1$. Timestamped edges~$(\vertexA,\vertexB;t)$ are represented as nodes~$\vertexA\vertexB^t$.}
    \label{fig:aug-event-graph}
\end{figure}
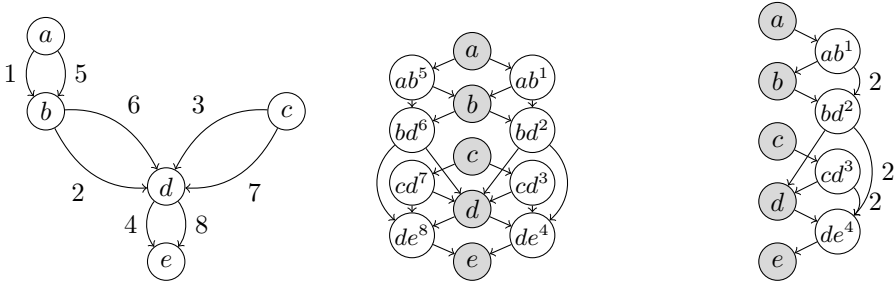

\begin{definition}[Augmented event graph]\label{def:aug-ev}
    Let~$\temporalGraph=(\vertices,\temporalEdges)$ be a temporal graph with event graph~$\eventGraph=(\temporalEdges,\eventEdges)$.
    The \emph{augmented event graph} is the static, directed, node-labeled graph~$\augmentedGraph=(\augmentedVertices,\augmentedEdges,\aLabel)$ with $\augmentedVertices = \vertices \cup \temporalEdges$, $\augmentedEdges = \eventEdges \cup \outEdges \cup \inEdges$ and
    \begin{align*}
        \aLabel(\vertex) &= \begin{cases}
        0 & \text{if }\vertex \in \vertices,\\
        1 & \text{if }\vertex \in \temporalEdges,
        \end{cases}\\
    \outEdges & = \{ (\vertexA,(\vertexA,\vertexB;t)) \mid (\vertexA,\vertexB;t)\in\temporalEdges \}, \qquad \inEdges = \{ ((\vertexA,\vertexB;t),\vertexB) \mid (\vertexA,\vertexB;t)\in\temporalEdges \}.\\
    \end{align*}
\end{definition}

\begin{theorem}\label{th:iso-equiv}
    Let~$\temporalGraph_1$ and~$\temporalGraph_2$ be temporal graphs with corresponding augmented event graphs~$\augmentedGraph_1$ and~$\augmentedGraph_2$.
    Then the following are equivalent:
    \begin{compactenum}[(i)]
    \item $\temporalGraph_1$ and $\temporalGraph_2$ are time-respecting path isomorphic.
    \item $\temporalGraph_1$ and $\temporalGraph_2$ are consistent event graph isomorphic.
    \item $\augmentedGraph_1$ and~$\augmentedGraph_2$ are isomorphic.
    \end{compactenum}
\end{theorem}

\begin{proof}[Proof sketch]
The equivalence of (i) and (ii) follows from two observations: first,
node-consistent edge mappings preserve paths (by definition), and
second, time-respecting paths of length~$\geq 2$ correspond to paths
in the temporal event graph. Path-preservation between temporal event
graphs is therefore equivalent to graph isomorphism between them.
The equivalence of (ii) and (iii) follows from the construction of
$G^{aug}$: the edges $E^{out}$ and $E^{in}$ encode the node consistency
property in the topology of $G^{aug}$, so that any isomorphism between
the augmented event graphs decomposes into a node-consistent pair
$(\pi_V, \pi_E)$, and conversely. (Full proof in~\Cref{app:equiv})
\end{proof}

\subparagraph{Comparison with Other Isomorphism Definitions}
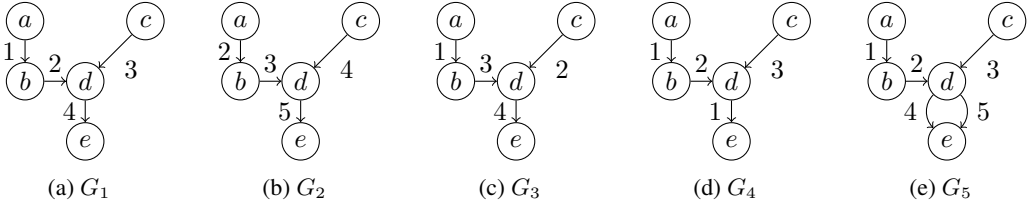
\begin{figure*}
\centering
\begin{subfigure}[b]{0.18\textwidth}
\centering
\input{fig/example}
\caption{$G_1$}
\end{subfigure}
\hfill
\begin{subfigure}[b]{0.18\textwidth}
\centering
\input{fig/example2}
\caption{$G_2$}
\end{subfigure}
\hfill
\begin{subfigure}[b]{0.18\textwidth}
\centering
\input{fig/example3}
\caption{$G_3$}
\end{subfigure}
\hfill
\begin{subfigure}[b]{0.18\textwidth}
\centering
\input{fig/example4}
\caption{$G_4$}
\end{subfigure}
\hfill
\begin{subfigure}[b]{0.18\textwidth}
\centering
\input{fig/example5}
\caption{$G_5$}
\end{subfigure}
\caption{Example illustrating different temporal graph isomorphism definitions for maximum waiting time $\delta=2$. Edges are labeled with timestamps. $G_1$ is time-concatenated isomorphic to~$G_2$, consistent event graph isomorphic to~$G_2$ and~$G_3$, and time-aggregated isomorphic to~$G_2$, $G_3$ and~$G_4$.}
\label{fig:iso-example}
\vspace{-10pt}
\end{figure*}

We reduced our notion of time-respecting path isomorphism to graph isomorphism on a special static representation of the temporal graph, namely the augmented event graph.
We now compare this to isomorphism notions that use different static representations, namely the time-aggregated and time-concatenated static graphs.

\begin{definition}[Time-aggregated/time-concatenated isomorphism]
Let~$\temporalGraph_1$ and~$\temporalGraph_2$ be two temporal graphs with the corresponding time-aggregated graphs~$\aggrGraph_1$ and~$\aggrGraph_2$ and the time-concatenated graphs~$\concGraph_1$ and~$\concGraph_2$.
We say that~$\temporalGraph_1$ and~$\temporalGraph_2$ are \emph{time-aggregated isomorphic} if~$\aggrGraph_1$ and~$\aggrGraph_2$ are isomorphic, and that they are \emph{time-concatenated isomorphic} if~$\concGraph_1$ and~$\concGraph_2$ are isomorphic.
\end{definition}

An equivalent representation of the time-concatenated static graph is a labeled multi-graph in which timestamps are treated as edge labels.
Hence, time-concatenated isomorphism is equivalent to the notions of isomorphism considered by~\citet{GaoRibeiro22} and~\citet{souza2022provably}.
Furthermore, it is similar to the \emph{timewise isomorphism} introduced by~\citet{walkega2024expressive}.
The latter was defined for temporal graphs with node labels that may change over time, which are not included in our model.
In~\Cref{app:compare}, we show that both notions are equivalent for temporal graphs without node labels.
The following theorem shows that consistent event graph isomorphism is stricter than time-aggregated isomorphism, but less strict than time-concatenated isomorphism.

\begin{theorem}\label{th:strictness}
    Let~$\temporalGraph_1=(\vertices_1,\temporalEdges_1)$ and~$\temporalGraph_2=(\vertices_2,\temporalEdges_2)$ be temporal graphs with time-aggregated graphs~$\aggrGraph_1$ and~$\aggrGraph_2$ and time-concatenated graphs~$\concGraph_1$ and~$\concGraph_2$.
    Then the following holds:
    \begin{compactenum}[(i)]
    \item If there exists an isomorphism~$\vertexMapping\colon\vertices_1\to\vertices_2$ between~$\concGraph_1$ and~$\concGraph_2$, then~$\edgeMapping \colon \temporalEdges_1 \to \temporalEdges_2$ with
    \begin{align*}
        \edgeMapping(\vertexA,\vertexB;t) =&\ (\mapping(\vertexA),\mapping(\vertexB);\minTime(\temporalGraph_2) - \minTime(\temporalGraph_1) + t)\\
        &\ \forall (\vertexA,\vertexB;t)\in\temporalEdges_1
    \end{align*}
    is a consistent event graph isomorphism.
    \item If there exists a consistent event graph isomorphism~$\edgeMapping\colon\temporalEdges_1\to\temporalEdges_2$, then the induced node mapping~$\vertexMapping\colon\vertices_1\to\vertices_2$ is a graph isomorphism between~$\aggrGraph_1$ and~$\aggrGraph_2$.
    \end{compactenum}
\end{theorem}

\begin{proof}[Proof sketch]
Claim~(i) follows directly from the definitions: a node bijection
preserving the multiset of timestamps per edge induces a node-consistent,
time-preserving edge mapping. For claim~(ii), we use that node
consistency and bijectivity of $\pi_E$ imply $|T_1(u,v)| =
|T_2(\pi_V(u), \pi_V(v))|$ for all node pairs, which is precisely
what time-aggregated isomorphism requires. (Full proof in~\Cref{app:compare}).
\end{proof}

Time-concatenated isomorphism preserves the values of all timestamps (aside from a constant offset).
By contrast, time-aggregated isomorphism ignores the timestamps altogether.
Instead, it only considers the static topology and the number of temporal edges between each node pair.
Consistent event graph isomorphism lies between the two:
Consider edges $\edgeA=(\vertexA,\vertexB;t)$ and~$\edgeA'=(\vertexA',\vertexB';t')$.
If there is a time-respecting path that includes~$\edgeA$ before~$\edgeA'$, then~$t < t'$ is enforced.
If there is no time-respecting path that includes both edges, then the relative order of~$t$ and~$t'$ is irrelevant.

\Cref{fig:iso-example} illustrates the differences. $G_1$ and~$G_3$ are
consistent event graph isomorphic but not time-concatenated isomorphic, as
the flipped timestamps of~$(b,d)$ and~$(c,d)$ lie on no common time-respecting
path. $G_1$ and~$G_4$ are time-aggregated but not consistent event graph
isomorphic, since~$(b,d;2),(d,e;4)$ is time-respecting only in~$G_1$.

\section{Message Passing in Event Graphs}
\label{sec:messagepassing}
We use the equivalency of time-respecting path isomorphism to static isomorphism on the augmented event graph to derive a message-passing GNN architecture for temporal graphs and characterize its expressive power.
Note that the augmented event graph is directed, even if the underlying temporal graph is undirected.
Edge directions are crucial because they represent the arrow of time, which is why we use the directed GNN Dir-GNN~\citep{Rossi23}.
It iteratively computes embeddings~$f^{(k)}(v)$ for each node~$v$ at layer~$k$.
This is done by aggregating embeddings of its neighbors at layer~$k-1$, using ~$\aggrIn{k}$ for incoming neighbors and~$\aggrOut{k}$ for outgoing neighbors.
A function~$\com{k}$ combines these with the previous embedding of~$v$ to a new embedding.
Formally, we have

\begin{align*}
    f^{(0)}(v) =&\enc(\nodeLabel(v)),\\
        f^{(k)}(v) =&\com{k}\left(f^{(k-1)}(v)\right.,
    \aggrIn{k}(\multiset{f^{(k-1)}(u) \mid u\in \inNeighbors(v)})
    ,\left.\aggrOut{k}(\multiset{f^{(k-1)}(u) \mid u\in \outNeighbors(v)})\right)
\end{align*}

where~$\com{k}$, $\aggrIn{k}$ and~$\aggrOut{k}$ are learnable.
The initial node embeddings are obtained by applying an injective encoding function~$\enc$ to the node labels.
To obtain a representation of the entire graph, we combine embeddings~$f^{(k)}(v)$ of all nodes~$v$ in final layer~$k$ with an injective readout function.

Our proposed GNN architecture simply applies Dir-GNNs to the augmented event graph.
By using the augmented event graph, this approach is specifically tailored towards detecting time-respecting path isomorphism.
We give an informal analysis of its expressive power (further details in~\Cref{app:wl-mp-aug}):
A model~$M_1$ is \emph{at least as expressive} as another model~$M_2$ if~$M_1$ distinguishes all node pairs that are distinguished by~$M_2$.
If the reverse is also true, they are \emph{equally as expressive}.
Otherwise, $M_1$ is \emph{strictly more expressive}.
~\citet{Rossi23} prove that if $\com{k}$, $ \aggrIn{k}$ and $ \aggrOut{k}$ are injective, Dir-GNN has the same expressive power as a directed version of the 1-WL test, called D-WL.
Both are strictly more expressive than undirected 1-WL and GNNs, i.e., there are graphs in which Dir-GNNs and D-WL distinguish more nodes than 1-WL and GNNs.

We finally note the similarity of our approach to the temporal WL graph kernel proposed by~\citet{OettershagenKMM20}.
Their approach applies 1-WL to the \emph{directed line graph expansion}, which is equivalent to applying D-WL to the (unaugmented) event graph and does not take into account the node consistency property (\Cref{def:consistent}).
Furthermore, our TGNN architecture offers more flexibility as it can learn the combination and aggregation functions.

\section{Experimental Evaluation}
\label{sec:results}

We experimentally evaluate our theoretical findings, focusing on the following research questions:

\begin{compactitem}
    \item[\bfseries (RQ1)] Does the message passing scheme derived in \Cref{sec:isomorphism} allow to distinguish non-isomorphic temporal graphs in a supervised graph classification task?

    \item[\bfseries (RQ2)] How does our model compare to existing baseline (temporal) GNNs in empirical datasets?
    \item[\bfseries (RQ3)] How does the strength of temporal patterns expressed in time-respecting paths influence the performance of TGNN architectures?

\end{compactitem}

Following standard practice in WL-based expressivity analyses~\citep{Morris2019,morris2020weisfeiler,BouritsasFZB23,OettershagenKMM20}, we evaluate our approach in a \emph{temporal graph classification} setting. This evaluation paradigm is the established methodology for expressivity analysis: it allows performance differences to be attributed directly to a model's ability to distinguish structural differences, rather than to confounding factors such as graph size or static topology.  We follow the same principled approach. We stress that structural distinguishability is a prerequisite for \emph{any} downstream task: a model that cannot distinguish relevant causal structures will be fundamentally limited regardless of the specific prediction objective.

We address binary classification, i.e., given $n$ temporal graphs $\temporalGraph_i$ ($i=1, \ldots, n$), we want to learn a classifier $C\colon \{ \temporalGraph_i \} \rightarrow \{0, 1\}$, with balanced ground truth classes assigned.
Importantly, for a $\temporalGraph_i$ we predict a \emph{single class} rather than multiple classes for different times $t$ in the evolution of $\temporalGraph_i$.
Details of the training procedure and hyperparameters are in \Cref{app:experiments}.

In each experiment, all $n$ temporal graphs have identical time-aggregated static graphs but differ only in terms of their \emph{causal topology}, i.e., which nodes can influence each other via time-respecting paths.

To address {\bf RQ1}, we use five empirical temporal graphs from biological and social systems with high-resolution timestamps, which have previously been shown to exhibit non-trivial patterns in their causal topology. Statistics of these datasets, including the maximum waiting time $\delta$ used for each, are reported in \Cref{tab:dataset_stats} (\Cref{app:experiments}).
To create a classification task from a given empirical temporal graph, we create multiple copies of the graph.
In each copy, we modify the edge timestamps by applying one or several randomization procedures, which differ in the degree to which they preserve certain features while destroying others.
For the graphs in Class 1, the applied procedures preserve most (but not all) of the time-respecting path structure of the original graph, but the values and the relative ordering of the timestamps may change drastically.
The graphs in Class 2 undergo an additional shuffling procedure that destroys the time-respecting path structure.
To solve this task, a TGNN must learn to distinguish between timestamp differences that affect time-respecting paths and those that do not.
Details on datasets, graph sizes and shuffling procedures are in \Cref{app:experiments}.

\begin{table*}[htb]
    \centering
    \scalebox{0.87}{
    \begin{tabular}{llcccccc}
        \toprule
        \textbf{Data} & &\textbf{Our Model}&\textbf{GAT}& \textbf{TGAT} & \textbf{TGN} & \textbf{CAW} & \textbf{PINT}  \\
        \midrule
        ants-1-1 &\citep{blonder2011time}& \bfseries{1.00 $\pm$ 0.00} &$0.52\pm 0.04$ &$0.54 \pm 0.05$& $0.86 \pm 0.03$ &$0.94 \pm 0.01$&$0.86 \pm 0.03$\\
        ants-1-2 &\citep{blonder2011time}&\bfseries{1.00 $\pm$ 0.00}&$0.54\pm 0.02$&$0.54 \pm 0.04$ &$0.85 \pm 0.03$  & $0.91 \pm 0.02$ &$0.83 \pm 0.04$\\
        sp-workplace &\citep{genois2018can}& \bfseries{1.00 $\pm$ 0.00} &$0.53\pm 0.03$ &$0.55 \pm 0.04$ &$0.75 \pm 0.04$ & $0.92 \pm 0.02$ & $0.83 \pm 0.04$ \\
        sp-hospital &\citep{vanhems2013estimating}&\bfseries{1.00 $\pm$ 0.00} &$0.53\pm 0.04$&$0.55 \pm 0.03$ & $0.75 \pm 0.04$ &$0.98 \pm 0.01$ &  $0.95 \pm 0.01$\\
        eu-email-dept2 &\citep{paranjape2017motifs}& \bfseries{1.00 $\pm$ 0.00} &$0.54\pm 0.07$&$0.52 \pm 0.03$ &$0.74 \pm 0.04$ &$0.76 \pm 0.03$ &$0.91 \pm 0.02$\\
        \bottomrule
    \end{tabular}
    }
    \caption{\small Mean accuracy for binary classification in real-world temporal graphs. For each dataset, we split the timeline into windows of 500 timestamps, generate 400 graphs per window via the shuffling procedures
    (\Cref{app:experiments}), performed 10 runs per window, and report the average over all runs.}
    \label{tab:rq2}
\end{table*}
\vspace{-.25cm}

\textbf{Discussion}
Regarding {\bfseries RQ1}, the left column in \Cref{tab:rq2} shows the mean accuracy of our TGNN in classifying empirical temporal graphs ($100$ runs).
Our model perfectly distinguishes Class 1 graphs (preserved causal topology) from Class 2 (destroyed by shuffling), supporting our theoretical results.

To address {\bfseries RQ2}, we compare our model against five baseline models, including both static and temporal GNNs:
For a simple approach that closely follows the notion of timewise isomorphism, we apply the Graph Attention Network (GAT) \citep{velivckovic2017graph} to the time-concatenated static graph, where the timestamps of all edges are included as edge features.
Moreover, we evaluate the Temporal Graph Attention Network (TGAT) model \citep{xu2020inductive}, TGN~\citep{rossi2020temporalgraphnetworksdeep}, CAW \citep{wang2021inductive}, as well as PINT \citep{souza2022provably}.
We follow \citet{xu2020inductive} to adapt these models for temporal graph classification (details in \Cref{app:experiments}).

\paragraph{Discussion} Regarding {\bf RQ2}, \Cref{tab:rq2} shows the performance on the same task used for {\bf RQ1}.
GAT and TGAT do not perform significantly better than a random guess, despite following the stricter notion of timewise isomorphism. Among purely message passing-based architectures, TGN performs best, likely due to its memory module. CAW and PINT, which incorporate sampled temporal random walks, perform significantly better. Overall, message passing schemes based on timewise isomorphism struggle to identify which timestamp differences affect time-respecting paths, whereas our model achieves perfect performance by modeling this explicitly in the augmented event graph.

\begin{wrapfigure}{l}{5cm}

    \centering
    \includegraphics[width=0.38\textwidth]{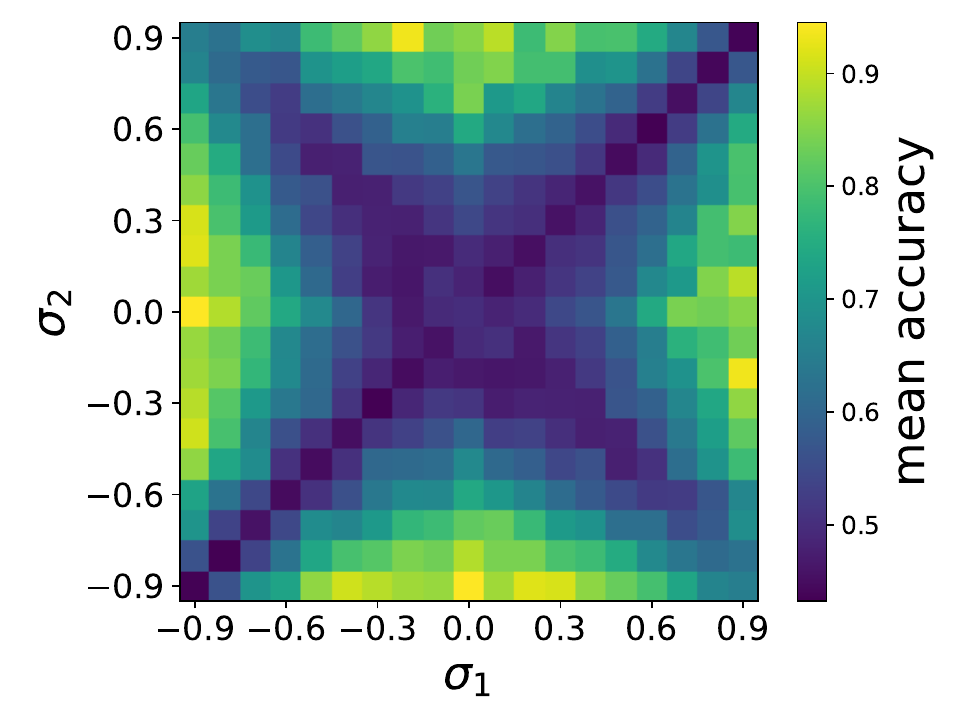}

   \caption{Mean classification accuracy for temporal graphs generated with $\sigma_1$ vs.\ $\sigma_2$ (25 runs each). \label{fig:classificationRQ3}}
   \vspace{-5pt}
\end{wrapfigure}

To address {\bfseries RQ3}, we quantify the strength of temporal patterns in terms of the causal topology.
We hypothesize that our model reliably distinguishes patterns in the time-respecting path structure provided they are sufficiently different from what is expected in a randomized version.
To test this, we use the model proposed by \citet{Scholtes2014_natcomm} to generate temporal graphs with two densely connected communities linked by few edges.
By varying a parameter $\sigma$, we selectively permute timestamps such that across-cluster time-respecting paths are over- ($\sigma > 0$) or underrepresented ($\sigma <0)$ compared to what we expect at random ($\sigma=0$).
For all $\sigma$ the time-aggregated topology, the frequency of all edges, as well as temporal activation patterns are identical.
Moreover, temporal graphs generated with values of $\sigma$ that differ more from zero exhibit stronger deviations in terms of time-respecting paths, which allows us to evaluate the sensitivity of our model.
Given two values $\sigma_1$ and $\sigma_2$, we form two classes for our classification task by generating graphs in Class 1 with $\sigma_1$ and those in Class 2 with $\sigma_2$.
In \Cref{app:experiments} we provide details on this model. For computational efficiency, we apply our message passing scheme to a compressed version of the augmented event graph (\Cref{app:compress}); on graphs generated by this model, the compression is lossless.

\paragraph{Discussion} Regarding {\bf RQ3}, \Cref{fig:classificationRQ3} shows mean accuracies of our TGNN for the experiment described above where the two classes are defined by different $\sigma_1$ and $\sigma_2$ (standard deviation in \Cref{fig:classification:stdev} in \Cref{sec:addresults}).
The results confirm that our model reliably distinguishes graphs with random patterns from non-random patterns, independent of whether cross-cluster time-respecting paths are over- ($\sigma>0$) or under-represented ($\sigma<0$).
We further find that the mean accuracy clearly depends on the strength of the pattern, quickly reaching accuracies close to one as $\sigma$ is moved away from zero.
This supports our hypothesis that our model reliably distinguishes temporal graphs provided their causal topologies are sufficiently different, as captured by different values of~$\sigma$.

\Cref{app:ablation} contains an \textbf{ablation study}, testing which components of our architecture are necessary.

\section{Conclusion}
\label{sec:conclusion}

We theoretically investigate the expressivity of temporal graph neural networks (TGNN).
We introduce a natural generalization of graph isomophism to temporal graphs by considering how the arrow of time shapes time-respecting paths and thus the causal topology of temporal graphs.
We show that this isomorphism can be heuristically tested by applying the directed and labeled Weisfeiler-Leman algorithm to augmented temporal event graphs.
This suggests that neural message passing is expressive enough to distinguish temporal graphs with identical static topology but different time-respecting paths.
An evaluation of our model in empirical and synthetic temporal graphs shows superior performance compared to existing models, which are based on stricter isomorphism definitions.
This supports arguments in \citet{Morris24a}, which highlights a complex relationship between expressivity and model performance.

\paragraph{Limitations and Open Issues}

Our work evaluates expressivity via temporal graph classification—the standard paradigm established by seminal works on GNN expressivity and graph kernels \citep{Morris2019,morris2020weisfeiler,OettershagenKMM20,BouritsasFZB23}
None of those works address node classification or link prediction, as these tasks do not directly correspond to the underlying notion of graph isomorphism. However, extending our framework to such downstream tasks is an interesting future work.
As noted in \citep{qarkaxhija2022bruijn,Heeg2025}, capturing causal topology can improve performance there, and our theoretical framework provides a foundation to analytically explain such gains.
We also did not compare our definition of temporal graph isomorphism to the notion of isomorphism on temporal computation trees, which is the basis for \citet{souza2022provably}. This interesting open question is an example for the research avenues opened by our work.

\section*{Acknowledgement}
Funded by the Deutsche Forschungsgemeinschaft (DFG, German
Research Foundation) – 459420781 and by the Lamarr Institute for Machine Learning and Artificial Intelligence, Bonn, Germany.

\bibliography{icml2026_conference}
\bibliographystyle{plainnat}

\newpage

\appendix
\section{Equivalence of Temporal Graph Isomorphism Notions}
\label{app:equiv}
In the following we give the proof of \Cref{th:iso-equiv}.
\begin{theorem*}
    Let~$\temporalGraph_1=(\vertices_1,\temporalEdges_1)$ and~$\temporalGraph_2=(\vertices_2,\temporalEdges_2)$ be two temporal graphs with corresponding augmented event graphs~$\augmentedGraph_1=(\augmentedVertices_1,\augmentedEdges_1,\aLabel_1)$ and~$\augmentedGraph_2=(\augmentedVertices_2,\augmentedEdges_2,\aLabel_2)$.
    Then the following statements are equivalent:
    \begin{compactenum}[(i)]
    \item $\temporalGraph_1$ and $\temporalGraph_2$ are time-respecting path isomorphic.
    \item $\temporalGraph_1$ and $\temporalGraph_2$ are consistent event graph isomorphic.
    \item $\augmentedGraph_1$ and~$\augmentedGraph_2$ are isomorphic.
    \end{compactenum}
\end{theorem*}

We begin by showing the equivalence of~(i) and~(ii):
\begin{proof}
    Let~$\vertexMapping\colon\vertices_1\to\vertices_2$ and~$\edgeMapping\colon\temporalEdges_1\to\temporalEdges_2$ be a node and edge mapping, respectively.
    Let~$\aPath=(\vertex_0,\edgeA_1,\vertex_1,\dots,\edgeA_{k},\vertex_k)$ be an alternating sequence of nodes and timestamped edges in~$\temporalGraph_1$.
    We denote the corresponding sequence in~$\temporalGraph_2$ that is induced by~$\vertexMapping$ and~$\edgeMapping$ as~$\mapping(\aPath)=(\vertexMapping(\vertex_0),\edgeMapping(\edgeA_1),\vertexMapping(\vertex_1),\dots,\edgeMapping(\edgeA_{k}),\vertexMapping(\vertex_k))$.
    We say that~$\vertexMapping$ and~$\edgeMapping$ are \emph{path-preserving} between~$\temporalGraph_1$ and~$\temporalGraph_2$ if for each sequence~$\aPath$ as defined above, $p$ is a path in~$\temporalGraph_1$ if and only if~$\mapping(\aPath)$ is a path in~$\temporalGraph_2$.
    It is easy to see that~$\vertexMapping$ and~$\edgeMapping$ are path-preserving between~$\temporalGraph_1$ and~$\temporalGraph_2$ if and only if $\edgeMapping$ is node-consistent with~$\vertexMapping$.

    Assume therefore that~$\vertexMapping$ and~$\edgeMapping$ are path-preserving between~$\temporalGraph_1$ and~$\temporalGraph_2$.
    We show that~$\edgeMapping$ is a graph isomorphism between the temporal event graphs~$\eventGraph_1$ and~$\eventGraph_2$ if and only if it is \emph{time-preserving}, i.e., a path~$\aPath$ in~$\temporalGraph_1$ is time-respecting iff~$\mapping(\aPath)$ is time-respecting in~$\temporalGraph_2$.
    If~$k=1$, this holds trivially because all paths of length~$1$ are time-respecting.
    If~$k \geq 2$, then~$\aPath$ is time-respecting if and only if~$(\edgeA_1,(\edgeA_1,\edgeA_2),\edgeA_2,\dots,(\edgeA_{k-1},\edgeA_{k}),\edgeA_{k})$ is a path in~$\eventGraph_1$.
    Hence, $\edgeMapping$ is time-preserving if and only if it is path-preserving between~$\eventGraph_1$ and~$\eventGraph_2$.
    Because the event graphs are unlabeled, this is the case if and only if~$\edgeMapping$ is a graph isomorphism by \Cref{obs:path-isomorphism}.
\end{proof}

Next, we show the equivalence of~(ii) and~(iii):
\begin{proof}
Let~$\mapping\colon\augmentedVertices_1\to\augmentedVertices_2$ be an isomorphism between~$\augmentedGraph_1$ and~$\augmentedGraph_2$.
Because~$\mapping$ preserves the node labels, it can be decomposed into bijective mappings~$\vertexMapping\colon\vertices_1\to\vertices_2$ and~$\edgeMapping\colon\temporalEdges_1\to\temporalEdges_2$.
Then~$\edgeMapping$ is an isomorphism between~$\eventGraph_1$ and~$\eventGraph_2$ because these are subgraphs of~$\augmentedGraph_1$ and~$\augmentedGraph_2$, respectively.
Consider an edge~$\edgeA=(\vertexA,\vertexB;t)\in\temporalEdges_1$.
By construction, $\augmentedGraph_1$ includes the edges~$(\vertexA,\edgeA)\in\outEdges_1$ and~$(\edgeA,\vertexB)\in\inEdges_1$.
Because~$\mapping$ is an isomorphism, it follows that~$(\vertexMapping(\vertexA),\edgeMapping(\edgeA))\in\outEdges_2$ and~$(\edgeMapping(\edgeA),\vertexMapping(\vertexB))\in\inEdges_2$.
Then it follows by construction of~$\augmentedGraph_2$ that~$\edgeMapping(\edgeA)=(\vertexMapping(\vertexA),\vertexMapping(\vertexB);t')$ for some~$t'\in\mathbb{N}$.

Conversely, let~$\edgeMapping\colon\temporalEdges_1\to\temporalEdges_2$ be a consistent event graph isomorphism between~$\temporalGraph_1$ and~$\temporalGraph_2$, and let~$\vertexMapping\colon\vertices_1\to\vertices_2$ be the induced node mapping such that
\[ \forall (\vertexA,\vertexB;t) \in \temporalEdges_1 \quad \exists t'\colon  \edgeMapping(\vertexA,\vertexB;t)=(\vertexMapping(\vertexA), \vertexMapping(\vertexB); t'). \]
Then~$\edgeMapping$ and~$\vertexMapping$ can be combined into a bijective mapping~$\mapping\colon\augmentedVertices_1\to\augmentedVertices_2$.
We show that~$\mapping$ is an isomorphism between~$\augmentedGraph_1$ and~$\augmentedGraph_2$.
By construction, $\mapping$ preserves the node labels.
For every pair of nodes~$x,y\in\augmentedVertices_1$ and every set of edges~$\edges'\in\{\eventEdges,\outEdges,\inEdges\}$, we show that
\[ (x,y) \in \edges'_1 \quad\iff\quad (\mapping(x),\mapping(y)) \in \edges'_2.
\]
For~$\edges'=\eventEdges$, this follows from the fact that~$\edgeMapping$ is an isomorphism between~$\eventGraph_1$ and~$\eventGraph_2$.
We show the case~$\edges'=\outEdges$ (the case~$\edges'=\inEdges$ is symmetrical):
We have~$(x,y) \in \outEdges_1$ if and only if~$y=(x,\vertexB;t)$ for some~$\vertexB\in\vertices_1$ and~$t\in\mathbb{N}$.
We have~$\mapping(y)=\edgeMapping(y)=(\vertexMapping(x),\vertexMapping(\vertexB);t') = (\mapping(x),\mapping(\vertexB);t')$ for some~$t'\in\mathbb{N}$.
By definition of~$\outEdges$, we have~$(\mapping(x),\mapping(y)) \in \outEdges_2$.
\end{proof}

\section{Temporal Event Graph Compression}
\label{app:compress}
\begin{figure*}[ht!]
    \centering
    \begin{subfigure}[b]{0.3\textwidth}
    \centering
    \input{fig/example_double}

    \end{subfigure}
    \begin{subfigure}[b]{0.3\textwidth}
    \centering
    \input{fig/augmented}

    \end{subfigure}
    \begin{subfigure}[b]{0.3\textwidth}
    \centering
    \input{fig/augmented_compressed}

    \end{subfigure}
    \caption{The temporal graph~$\temporalGraph$ (left) and corresponding augmented event graph $\augmentedGraph$ (center) from~\Cref{fig:aug-event-graph}, as well as the compressed augmented event graph $\graph^\text{comp}$ (right).
    In~$\augmentedGraph$ and~$\graph^\text{comp}$, gray nodes have label~$0$ and white nodes have label~$1$. Timestamped edges~$(\vertexA,\vertexB;t)$ are represented as nodes~$\vertexA\vertexB^t$. In~$\graph^\text{comp}$, edge weights represent the number of connected components of~$\augmentedGraph$ in which the edge appears.}
    \label{fig:aug-event-graph-compress}
\end{figure*}

To improve the computational efficiency, we propose a compressed representation of the augmented event graph (see example in \Cref{fig:aug-event-graph-compress}).
This is based on the observation that the temporal event graph often contains many connected components that represent the same set of time-respecting paths, but with different timestamps.
Merging these components reduces the size of the graph without losing any information about the causal topology.

\begin{definition}[Compressed augmented event graph]
Let~$\temporalGraph=(\vertices,\temporalEdges)$ be a temporal graph with its event graph~$\eventGraph=(\temporalEdges,\eventEdges)$.
Let~$\mathcal C$ denote the set of connected components in~$\eventGraph$.
For each connected component~$C\in\mathcal C$ and each node pair~$\vertexA,\vertexB$, let~$t_1^C<\dots<t_{k_C}^C$ be the timestamps such that~$(\vertexA,\vertexB;t_i^C)$ is a node in~$C$, sorted in ascending order.
We define the graph~$\tau(C)$ by replacing each timestamped edge~$\edgeA=(\vertexA,\vertexB;t_i^C)$ in~$C$ with~$\tau(\edgeA)=(\vertexA,\vertexB;i)$, with edges between these nodes inherited from~$C$.
Two connected components~$C_1,C_2\in\mathcal C$ are \emph{equivalent} if~$\tau(C_1)=\tau(C_2)$ as static graphs.
We compress~$\eventGraph$ by replacing each equivalence class of~$\mathcal C$ with a single representative, in which each edge is weighted with the size of the class.
The \emph{compressed augmented event graph} is then built according to \Cref{def:aug-ev}.
\end{definition}

For reasons of computational efficiency, our notion of equivalence considers the relative order of timestamped edges between the same node pair.
We note that there are cases in which connected components are not considered equivalent even if their causal topology is the same.
In these cases, the compressed augmented event graph does not preserve all isomorphisms.

\Cref{fig:compr-counter} shows an example in which a consistent event graph isomorphism between two temporal graphs~$\temporalGraph_1$ and~$\temporalGraph_2$ is lost when compressing the temporal event graph.
In the temporal event graph of both~$\temporalGraph_1$ and~$\temporalGraph_2$, all four connected components are isomorphic, but there are two equivalence classes.
This is due to the relative order of the two~$(b,c)$ edges: in some components, the edge that is adjacent to~$(a,b)$ has an earlier timestamp than the other one, but in other components the order is flipped.
The cardinality of the equivalence classes differs between~$\eventGraph_1$ and~$\eventGraph_2$, and therefore the edge weights of the representatives in the compressed event graphs differs.
Therefore, there is no node-consistent isomorphism between the compressed event graphs, even though there is one between the uncompressed ones.

This phenomenon is due to the fact that the definition of equivalence between connected components considers the relative order of timestamped edges that connect the same node pair.
This is done to ensure that connected components can be tested for equivalence efficiently.
As \Cref{fig:compr-counter} shows, it is possible for two connected components to be isomorphic even if the relative order is different.
Hence, the notion of consistent event graph isomorphism becomes slightly stricter if the event graphs are compressed.
As a result, a WL test or GNN architecture using event graph compression may distinguish some graphs that should not be distinguished.
Nevertheless, this loss in precision may be acceptable in some cases because the compression can significantly reduce the size of the graph and thereby makes the GNN easier to train.

We note that the compressed augmented event graph is only stricter than the uncompressed variant in rare cases like illustrated in \Cref{fig:aug-event-graph-compress}, where two connected components of the temporal event graph share the same causal topology but differ in the relative order of timestamped edges between the same node pair. In our experiments, the synthetic cluster model does not produce such cases, so the compression is lossless in practice. We therefore use it for the synthetic experiments to reduce graph size, while the empirical experiments use the uncompressed variant.

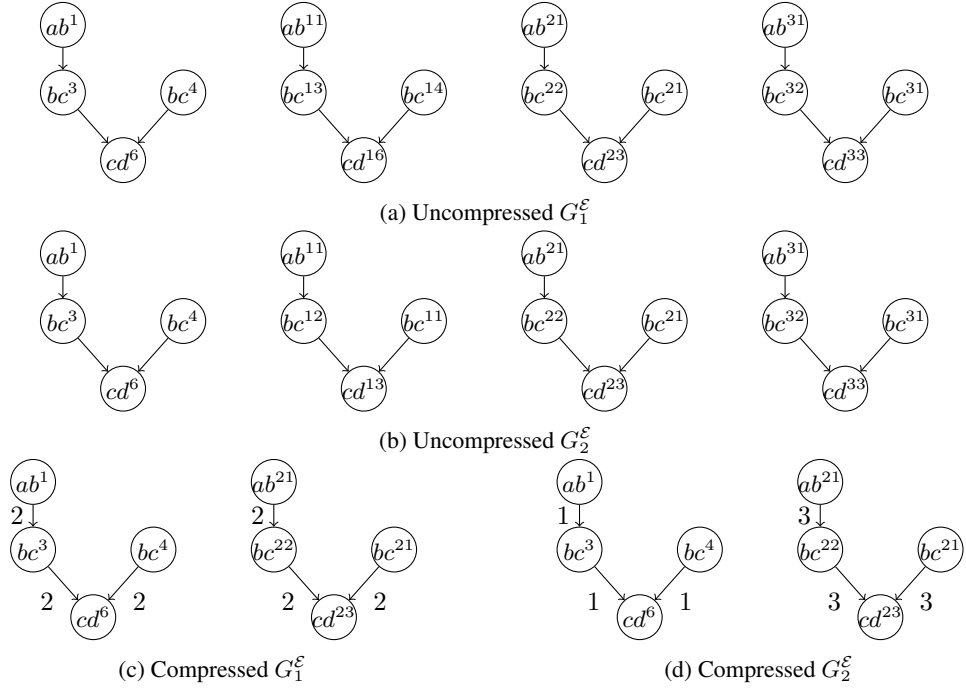
\begin{figure*}
    \centering
    \begin{subfigure}[b]{\textwidth}
    \centering
    \input{fig/compressed_counterexample}
    \caption{Uncompressed $\eventGraph_1$}
    \end{subfigure}
    \begin{subfigure}[b]{\textwidth}
    \centering
    \input{fig/compressed_counterexample_2}
    \caption{Uncompressed $\eventGraph_2$}
    \end{subfigure}
    \begin{subfigure}[b]{0.48\textwidth}
    \centering
    \input{fig/compressed_counterexample_3}
    \caption{Compressed~$\eventGraph_1$}
    \end{subfigure}
    \hfill
    \begin{subfigure}[b]{0.48\textwidth}
    \centering
    \input{fig/compressed_counterexample_4}
    \caption{Compressed~$\eventGraph_2$}
    \end{subfigure}
    \caption{An example of two temporal graphs~$\temporalGraph_1$ and~$\temporalGraph_2$ which are consistent event graph isomorphic, but there is no node-consistent isomorphism between the compressed event graphs.}
    \label{fig:compr-counter}
\end{figure*}

\section{Comparison with Other Graph Isomorphisms}
\label{app:compare}

In the following we give the proof of \Cref{th:strictness}.

\begin{theorem*}
    Let~$\temporalGraph_1=(\vertices_1,\temporalEdges_1)$ and~$\temporalGraph_2=(\vertices_2,\temporalEdges_2)$ be two temporal graphs with time-aggregated graphs~$\aggrGraph_1=(\vertices_1,\staticEdges_1,\aggrLabel_1)$ and~$\aggrGraph_2=(\vertices_2,\staticEdges_2,\aggrLabel_2)$ and time-concatenated graphs~$\concGraph_1=(\vertices_1,\staticEdges_1,\concLabel_1)$ and~$\concGraph_2=(\vertices_2,\staticEdges_2,\concLabel_2)$.
    Then the following holds:
    \begin{compactenum}[(i)]
    \item If there exists an isomorphism~$\vertexMapping\colon\vertices_1\to\vertices_2$ between~$\concGraph_1$ and~$\concGraph_2$, then~$\edgeMapping \colon \temporalEdges_1 \to \temporalEdges_2$ with
    \begin{align*}
        \edgeMapping(\vertexA,\vertexB;t) =&\ (\mapping(\vertexA),\mapping(\vertexB);\minTime(\temporalGraph_2) - \minTime(\temporalGraph_1) + t)\\
        &\ \forall (\vertexA,\vertexB;t)\in\temporalEdges_1
    \end{align*}
    is a consistent event graph isomorphism.
    \item If there exists a consistent event graph isomorphism~$\edgeMapping\colon\temporalEdges_1\to\temporalEdges_2$, then the induced node mapping~$\vertexMapping\colon\vertices_1\to\vertices_2$ is a graph isomorphism between~$\aggrGraph_1$ and~$\aggrGraph_2$.
    \end{compactenum}
\end{theorem*}

\begin{proof}
    Claim~(i) follows directly from the definitions.
    For claim~(ii), let~$\edgeMapping\colon\temporalEdges_1\to\temporalEdges_2$ be a consistent event graph isomorphism.
    For every node pair~$\vertexA,\vertexB\in\vertices_1$, we have
\begin{align*}
|\timestamps_1(\vertexA,\vertexB)| &= |\{(\vertexA,\vertexB;t) \in \temporalEdges_1\}|\\
    &= |\{\edgeMapping(\vertexA,\vertexB;t') \in \temporalEdges_2\}|\\
    &= |\{(\vertexMapping(\vertexA),\vertexMapping(\vertexB);t') \in \temporalEdges_2\}|\\
    &= |\timestamps_2(\vertexMapping(\vertexA),\vertexMapping(\vertexB))|.
\end{align*}
Here, we use the node consistency property and the fact that~$\edgeMapping$ is a bijection.
It follows that
\begin{align*}
 (\vertexA,\vertexB)\in\staticEdges_1 &\iff \timestamps_1(\vertexA,\vertexB) \neq \emptyset\\
 &\iff \timestamps_2(\vertexMapping(\vertexA),\vertexMapping(\vertexB)) \neq \emptyset\\
 &\iff (\vertexMapping(\vertexA),\vertexMapping(\vertexB))\in\staticEdges_2
\end{align*}
and
\begin{align*}
     \aggrLabel_1(\vertexA,\vertexB) = |\timestamps_1(\vertexA,\vertexB)| &= |\timestamps_2(\vertexMapping(\vertexA),\vertexMapping(\vertexB))|\\ &= \aggrLabel_2(\vertexMapping(\vertexA),\vertexMapping(\vertexB)).\qedhere
\end{align*}
\end{proof}

\subsection{Timewise Isomorphism with Time-Variant Node Labels}
\citet{walkega2024expressive} present isomorphism definitions for snapshot-based temporal graphs in which the nodes have labels that can vary over time.

\begin{definition}[Snapshot-based temporal graph]
    A snapshot-based temporal graph is a sequence~$(\graph_1,t_1),\dots,(\graph_n,t_n)$ with~$t_1 < \dots < t_n$ and~$\graph_i=(\vertices,\edges_i,c_i)$, where~$c_i \colon \vertices \to \nodeLabels$ is a node labeling.
\end{definition}

\begin{definition}[Timewise isomorphism]
Two snapshot-based temporal graphs~$(\graph_1,t_1),\dots,(\graph_n,t_n)$ and~$(\hat{\graph}_1,\hat{t}_1),\dots,(\hat{\graph}_{\hat{n}},\hat{t}_{\hat{n}})$ are timewise isomorphic if~$n=\hat{n}$, $t_i - t_1 = \hat{t}_i - \hat{t}_1$ for all~$1 \leq i \leq n$, and there is a bijection~$\mapping \colon \vertices \to \vertices'$ such that $\mapping$ is an isomorphism between~$\graph_i$ and~$\hat{\graph}_i$ for all~$1 \leq i \leq n$.
\end{definition}

We consider the case that there are no node labels.
In this case, a snapshot-based temporal graph~$(\graph_1,t_1),\dots,(\graph_n,t_n)$ has an equivalent temporal graph~$\temporalGraph=(\vertices,\temporalEdges)$ according to our definition with
\[ \temporalEdges = \{ (\vertexA,\vertexB;t_i) \mid (\vertexA,\vertexB) \in \edges_i, 1 \leq i \leq n \}. \]

The following theorem shows that if there are no node labels, timewise isomorphism is equivalent to isomorphism of the time-concatenated static graphs.

\begin{theorem}
    Let~$(\graph_1,t_1),\dots,(\graph_n,t_n)$ and~$(\hat{\graph}_1,t_1),\dots,(\hat{\graph}_n,t_n)$ be two snapshot-based temporal graphs without node labels, $E_i\not= \emptyset$, and let~$\mapping\colon\vertices\to\hat{\vertices}$ be a node bijection.
    Then~$\mapping$ is an isomorphism between the time-concatenated static graphs~$\concGraph = (\vertices,\staticEdges,\concLabel)$ and~$\hat{\concGraph} = (\hat{\vertices},\hat{\staticEdges},\hat{\concLabel})$ iff~$n = \hat{n}$, $t_i - t_1 = \hat{t}_i - \hat{t}_1$ for all~$1 \leq i \leq n$, and~$\mapping$ is an isomorphism between~$\graph_i$ and~$\hat{\graph}_i$ for all~$1 \leq i \leq n$.
\end{theorem}
\begin{proof}
For each node pair~$\vertexA,\vertexB\in\vertices$, we have
\begin{align*}
\timestamps(\vertexA,\vertexB) &= \{ t - \minTime(\temporalGraph) \mid (\vertexA,\vertexB;t) \in \temporalEdges \}\\
&= \{ t_i - t_1 \mid (\vertexA,\vertexB) \in \edges_i, 1 \leq i \leq n \}.
\end{align*}
We note that~$\mapping$ is an isomorphism between~$\concGraph$ and~$\hat{\concGraph}$ iff~$\timestamps(\vertexA,\vertexB) = \hat{\timestamps}(\mapping(\vertexA),\mapping(\vertexB))$ holds for all~$\vertexA,\vertexB\in\vertices$.
If~$n=\hat{n}$, $t_i - t_1 = \hat{t}_i - \hat{t}_1$ for all~$1 \leq i \leq n$, and~$\mapping$ is an isomorphism between~$\graph_i$ and~$\hat{\graph}_i$ for all~$1 \leq i \leq n$, then it is easy to see that~$\timestamps(\vertexA,\vertexB) = \hat{\timestamps}(\mapping(\vertexA),\mapping(\vertexB))$ for all~$\vertexA,\vertexB\in\vertices$.
On the other hand, assume that~$\timestamps(\vertexA,\vertexB) = \hat{\timestamps}(\mapping(\vertexA),\mapping(\vertexB))$ for all~$\vertexA,\vertexB\in\vertices$.
We have
\begin{align*}
(\vertexA,\vertexB) \in \edges_i &\iff t_i - t_1 \in \timestamps(\vertexA,\vertexB)\\
&\iff t_i - t_1 \in \hat{\timestamps}(\mapping(\vertexA),\mapping(\vertexB))\\
&\iff \exists j\colon \hat{t}_j - \hat{t}_1 = t_i - t_1\\
&\phantom{\iff \exists j\colon}\land \hat{t}_j - \hat{t}_1 \in \hat{\timestamps}(\mapping(\vertexA),\mapping(\vertexB))\\
&\iff \exists j\colon \hat{t}_j - \hat{t}_1 = t_i - t_1\\
&\phantom{\iff \exists j\colon}\land (\mapping(\vertexA),\mapping(\vertexB)) \in \hat{\edges}_j.
\end{align*}
Within each temporal graph, all timestamps are distinct from each other, so by the pigeonholing principle, we have~$n=\hat{n}$ and~$t_i - t_1 = \hat{t}_j - \hat{t}_1$ for all~$1 \leq i \leq n$.
Then it follows that~$\mapping$ is an isomorphism between~$\graph_i$ and~$\hat{\graph}_i$ for all~$1 \leq i \leq n$.
\end{proof}

\section{Expressive Power of Message Passing on the Augmented Event Graph}
\label{app:wl-mp-aug}

We formally analyze the expressive power of the message passing approach presented in~\Cref{sec:messagepassing}, which applies Dir-GNN to the augmented event graph.
\citet{Rossi23} propose a directed version of the 1-WL test, called D-WL.
It is a \emph{color refinement algorithm} that iteratively computes a node coloring~$c^{(t)}\colon \vertices \to \mathbb{N}$ for each iteration~$t \geq 0$ as follows:
\begin{align*}
c^{(0)}(v) = &\ \nodeLabel(v) \\
c^{(t)}(v) = &\ \hash\left(c^{(t-1)}(v),\right.\\
    &\ \multiset{c^{(t-1)}(u) \mid u\in \inNeighbors(v)},\\
    &\left.\ \multiset{c^{(t-1)}(u) \mid u\in \outNeighbors(v)}\right).
\end{align*}
Here, $\hash$ is an injective function.
Compared to the standard 1-WL test, D-WL takes edge directions into account by considering the multisets of node colors of the incoming and outgoing neighbors separately.

The node embeddings computed by GNNs can also be interpreted as node colorings, and thus GNNs can be seen as color refinement algorithms.
The ability of color refinement algorithms to distinguish nodes is captured by the following definition.

\begin{definition}[Expressivity]
    A color refinement algorithm~$c$ is at least as expressive as another color refinement algorithm~$\hat{c}$ if for all graphs~$\graph$, all iterations~$t \geq 0$ and all node pairs~$\vertexA,\vertexB$, it holds that~$\hat{c}^{(t)}(\vertexA) \neq \hat{c}^{(t)}(\vertexB)$ implies~$c^{(t)}(\vertexA) \neq c^{(t)}(\vertexB)$. If the reverse implication holds as well, $c$ and~$\hat{c}$ are equally as expressive. Otherwise, $c$ is strictly more expressive than~$\hat{c}$.
\end{definition}

\citet{Morris2019} show that undirected GNNs are equally as expressive as undirected 1-WL if the aggregation and combination functions are injective.
\citet{Rossi23} extend this result to the directed variants: Dir-GNN is equally as expressive as D-WL if $\com{t}$, $ \aggrIn{t}$ and $ \aggrOut{t}$ are injective.
Furthermore, both are strictly more expressive than undirected 1-WL and GNNs, i.e., there are graphs in which Dir-GNNs and D-WL distinguish more nodes than 1-WL and GNNs.
Because the augmented event graph is a static directed graph, these results automatically carry over to our setting:

\begin{theorem}
    If $\com{t}$, $ \aggrIn{t}$ and $ \aggrOut{t}$ are injective, Dir-GNN on the augmented event graph has the same expressive power as D-WL on the augmented event graph.
\end{theorem}

\section{Detailed Experimental Protocol}
\label{app:experiments}

In the following, we give a detailed overview of our experimental protocol described in \Cref{sec:results}.

\begin{table}[h]
\centering

\begin{tabular}{llrrrr}
\toprule
Dataset && $|V|$ & $|E|$ & $|E^\tau|$ & $\delta$ in sec. \\
\midrule
ants-1-1 &\cite{blonder2011time} &89 &947 & 1,911 & 30 \\
ants-1-2 &\cite{blonder2011time} & 72 & 862 & 1,820 & 30 \\
sp-workplace& \cite{genois2018can} & 92 &755 & 9,827 & 3,600 \\
sp-hospital &\cite{vanhems2013estimating} & 75 & 1,139 & 32,424 & 3,600 \\
eu-email-dept2 &\cite{paranjape2017motifs} & 162 & 1,772 & 46,772 & 3,600 \\
\bottomrule
\\
\end{tabular}
\caption{Statistics of empirical temporal graphs used in our evaluation. $|V|$ is the number of nodes, $|E^\tau|$ the number of timestamped edges, and $\delta$ the maximum waiting time used to construct the augmented event graph.}

\label{tab:dataset_stats}
\end{table}

\paragraph{Details on data sets} Detailed statistics on the empirical temporal graphs used in our experiments can be found in \cref{tab:dataset_stats}.
All data sets are properly credited in the table and publicly available online. The SocioPatterns data \texttt{sp-workplace} and \texttt{sp-hospital} are made available under a CC-0 1.0 license \footnote{\url{https://creativecommons.org/publicdomain/zero/1.0/}}.

\paragraph{Details on Shuffling Method for Empirical Temporal Graphs}

Starting from a given empirical temporal graph, we implement three different methods that randomly shuffle the timestamps of the edges in ways that preserve certain features but destroy others.
Method A shifts the timestamp of each edge by a random value between~$0$ and~$\delta/2$.
This mostly preserves the structure of time-respecting paths, but not entirely.
Method B identifies the connected components of the temporal event graph and reassigns the timestamps such that the relative values within each component are preserved but the temporal order of the components is changed, while ensuring a temporal gap of more than $\delta$ between consecutive components.
Because there are no time-respecting paths that cross different components, this means that the causal topology remains unchanged even as the global ordering of the events changes drastically.
Method C shuffles all timestamps randomly, which preserves the frequency of all edges but is highly likely to change time-respecting paths.
Note that all three methods leave the time-aggregated static graph unchanged.
With these methods, we build four graph classes.
To build Class 1a, Method A is applied to 100 different copies of the empirical temporal graph.
Class 2a is then constructed by applying Method C to every graph in Class 1a.
Finally, Classes 1b and 2b are built by applying Method B to every graph in Class 1a and 2a, respectively.
For the purpose of our supervised binary classification task, Classes 1a and 1b are grouped together into Class 1, and Classes 2a and 2b into Class 2.
The task is to distinguish Class 1, which mostly preserves the time-respecting paths of the original temporal graph, from Class 2, which does not.
Because Method B can introduce drastic changes in the ordering of the timestamps without affecting the time-respecting paths, this task requires distinguishing between timestamp differences that affect time-respecting paths and those that do not.

\paragraph{Details on Synthetic Cluster Model} We adopt the stochastic model proposed by \citet{Scholtes2014_natcomm}.
This model is based on a static graph with two strong communities, each consisting of a $k$-regular random graph with $k=3$ and $n_1=n_2=10$ nodes that are interconnected by two edges.
We randomly generate temporal graphs by simulating $500$ second-order random walks of length two.
Traversed edges are assigned consecutive timestamps $t$ and $t+1$ within each walk, with an additional time gap of $t+2$ separating different walks.
This leads to $250$ temporal graphs with $1000$ timestamped edges, each containing $500$ time-respecting paths of length two (for $\delta=1$).
A parameter $\sigma \in (-1,1)$ allows to tune random walk transition probabilities such that (i) for $\sigma<0$ time-respecting paths between \emph{different} communities are underrepresented compared to a shuffled temporal graph, (ii) for $\sigma>0$ time-respecting paths connecting different communities are \emph{overrepresented}, and (iii) for all $\sigma$ time-aggregated static graphs are identical (see  \citep{Scholtes2014_natcomm}).
This model generates random temporal graphs that share the same time-aggregated static topology, but whose causal topology differs in terms of how well nodes in different communities are connected via time-respecting paths.
We use this model to generate $250$ temporal graphs for different $\sigma$.
Given two parameter choices $\sigma_1$ and $\sigma_2$, we assign graphs generated with $\sigma_1$ to one class and those generated with $\sigma_2$ to the other class.

\paragraph{Implementation Details}
We used the open-source temporal graph learning library \texttt{pathpyG} to implement the random walk-based models for the empirical and synthetic temporal graphs explained above.
We then used pathpyG to generate augmented event graphs for all temporal graphs generated by the two models, with the dataset-specific values of $\delta$ reported in Table~\ref{tab:dataset_stats}. For the synthetic cluster model, we use $\delta = 1$.
We used the compressed augmented event graph for the experiments on the synthetic cluster model, but the (uncompressed) augmented event graph for the experiments on the empirical temporal graphs.

We implemented the TGNNs based on their respective reference implementations made available by the authors, specifically:
TGAT \cite{xu2020inductive}\footnote{\scriptsize \url{https://github.com/StatsDLMathsRecomSys/Inductive-representation-learning-on-temporal-graphs}}, TGN \cite{rossi2020temporalgraphnetworksdeep}\footnote{\scriptsize\url{https://github.com/twitter-research/tgn}}, CAW \cite{wang2021inductive} \footnote{\scriptsize\url{https://github.com/snap-stanford/CAW}}, PINT \cite{souza2022provably}\footnote{\scriptsize\url{https://github.com/AaltoPML/PINT/}}. These models are originally trained in a self-supervised manner for link prediction. We adopt this training procedure to obtain temporal node embeddings and subsequently use the learned embeddings as input to a downstream multi-layer perceptron (MLP) for the graph classification task. For the self-supervised tasks we use the default model parameters as given in the implementations.

For message passing in the augmented event graph, we use three convolutional layers (with layer widths being hyperparameters) using pyG's implementation of Dir-GNN (using the \texttt{GraphConv} layer) proposed by \citet{Rossi23}.
We use an \texttt{add} pooling layer and two dense linear classification layers, where layer widths are hyperparameters.

For all experiments, we address a balanced, binary temporal graph classification task, i.e., given $n$ temporal graphs $\temporalGraph_i$ ($i=1, \ldots, n$), we want to learn a classifier $C\colon \{ \temporalGraph_i \} \rightarrow \{0, 1\}$, with ground truth classes assigned as described above.
Importantly, for a $\temporalGraph_i$ we predict a \emph{single class} rather than multiple classes for different times $t$ in the evolution of $\temporalGraph_i$.

For the graph classification MLPs from the TGNNs and for GAT and our event graph neural network (EGNN) , we use a single output with sigmoid activation function in the final layer and train the model using binary cross entropy loss.

We use a one-hot encoding (OHE) of node labels in the augmented event graph (cf.\ $\nodeLabel$ in \Cref{def:aug-ev}) as node features for EGNN and an OHE of the nodes for GAT. In the other MLPs the node features are obtained by the self supervised link prediction of the TGNNs.

Since TGNNs produce time-dependent node representations, we consider multiple strategies to obtain graph-level inputs for the downstream MLP classifier. Specifically, we evaluate whether to use (i) all temporal node embeddings, (ii) only the last temporal embedding per node, or (iii) the mean over all temporal embeddings of each node. These node-level representations are then aggregated at graph level either via max pooling, mean pooling, or by directly concatenating node embeddings without additional pooling. The optimal combination of temporal selection and graph-level aggregation strategy is determined through hyperparameter tuning and reported in Tables~\ref{tab:hyperparameters}-\ref{tab:hyperparameters:tgn}.

For all experiments we ran the Adam optimizer, with 100 epochs with a 80/20 training/test split.
We performed a grid search to tune hyperparameters (learning rate, width of GNN and dense classification layers and the pooling). The optimal hyperparameters used to obtain the results are reported in Tables~\ref{tab:hyperparameters}-\ref{tab:hyperparameters:tgn}.

\begin{table*}
\centering
\begin{tabular}{|c|c|c|}
\hline
 & Our model & GAT\\
\hline
batch size & 200  & 50 \\
weight\_decay &0.0001  & 0 \\
learning rate & 0.001  & 0.001  \\
DirGNN layers & 32 $\rightarrow$ 16 $\rightarrow$ 8 & --\\
GatConv layers & --  & 8 $\rightarrow$ 64 $\rightarrow$ 64\\
dense layers & 8 $\rightarrow$ 4 $\rightarrow$ 1 & 64 $\rightarrow$ 1\\
pooling &  global add & global mean\\
\hline
node feature dim &  2 (OHE of node labels) & n (OHE of nodes)\\
tunable parameters &  3777 & 8n+5065\\
\hline
\end{tabular}
\medskip
\caption{Optimal hyperparameter values and number of tunable model parameters (i.e., model size) for our model and GAT.}
\label{tab:hyperparameters}
\end{table*}

\begin{table*}
    \centering
     \scalebox{0.8}{
    \begin{tabular}{|c|c|c|c|c|c|}
    \hline
         & ants-1-1&ants-1-2 & workplace& hospital& eu-email-dept2  \\
         \hline
        batch size & 1&1 &1  &1 &1 \\
        weight decay &0 &0 &0 &0  &0 \\
        learning rate &0.001 &0.001 &0.01 &0.001 & 0.001\\
        layers &64 $\rightarrow$ 128 $\rightarrow$ 1 &64 $\rightarrow$ 128$\rightarrow$ 64  $\rightarrow$ 1 &64 $\rightarrow$ 128 $\rightarrow$ 1&64 $\rightarrow$ 128 $\rightarrow$ 1 &64 $\rightarrow$ 128 $\rightarrow$ 64 $\rightarrow$ 1  \\
        aggr.\ node embeddings &mean  &- &-&mean &- \\
        pooling&- &- &-  &- &max \\
        \hline
        node feature dim &64&64&64&64&64 \\
        tunable parameters &826565 & 839494 &827589 &822085 &842310\\
        \hline
    \end{tabular}}
    \caption{Optimal hyperparameter values and number of tunable model parameters for CAW.}
    \label{tab:hyperparameters:caw}
\end{table*}

\begin{table*}
    \centering
     \scalebox{0.8}{
    \begin{tabular}{|c|c|c|c|c|c|}
    \hline
         & ants-1-1&ants-1-2 & workplace& hospital& eu-email-dept2  \\
         \hline
        batch size & 1&1 &1  &1 &1 \\
        weight decay &0 &0 &0 &0  &0 \\
        learning rate &0.001 &0.001 &0.01 &0.001 & 0.001\\
        layers &64 $\rightarrow$ 128 $\rightarrow$ 1 &64 $\rightarrow$ 128$\rightarrow$ 64  $\rightarrow$ 1 &64 $\rightarrow$ 128 $\rightarrow$ 1&64 $\rightarrow$ 128 $\rightarrow$ 1 &64 $\rightarrow$ 128 $\rightarrow$ 64 $\rightarrow$ 1  \\
        aggr.\ node embeddings &mean  &- &mean&- &- \\
        pooling&- &- &max  &- &- \\
        \hline
        node feature dim &64&64&64&64&64 \\
        tunable parameters &16336 & 28155 & 16472 &15741 &28529\\
        \hline
    \end{tabular}}
    \caption{Optimal hyperparameter values and number of tunable model parameters for PINT.}
    \label{tab:hyperparameters:pint}
\end{table*}

\begin{table*}
    \centering
     \scalebox{0.8}{
    \begin{tabular}{|c|c|c|c|c|c|}
    \hline
         & ants-1-1&ants-1-2 & workplace& hospital& eu-email-dept2  \\
         \hline
        batch size & 1&1 &1  &1 &1 \\
        weight decay &0 &0 &0 &0  &0 \\
        learning rate &0.01 &0.001 &0.001 &0.01 & 0.001\\
        layers &64 $\rightarrow$ 128 $\rightarrow$ 1 &64 $\rightarrow$ 128$\rightarrow$ 1 &64 $\rightarrow$ 128 $\rightarrow$ 1&64 $\rightarrow$ 128 $\rightarrow$ 1 &64 $\rightarrow$ 128 $\rightarrow$ 1  \\
        aggregation node embeddings &mean  &mean &last &mean&mean \\
        pooling&- &max &max  &max &max \\
        \hline
        node feature dim &64&64&64&64&64 \\
        tunable parameters &402243 & 402883 & 402755 &400003 &404291\\
        \hline
    \end{tabular}}
    \caption{Optimal hyperparameter values and number of tunable model parameters for TGAT.}
    \label{tab:hyperparameters:tgat}
\end{table*}

\begin{table*}
    \centering
     \scalebox{0.8}{
    \begin{tabular}{|c|c|c|c|c|c|}
    \hline
         & ants-1-1&ants-1-2 & workplace& hospital& eu-email-dept2  \\
         \hline
        batch size & 1&1 &1  &1 &1 \\
        weight decay &0 &0 &0 &0  &0 \\
        learning rate &0.001 &0.001 &0.001 &0.001 & 0.001\\
        layers &64 $\rightarrow$ 128 $\rightarrow$ 1 &64 $\rightarrow$ 128$\rightarrow$ 1 &64 $\rightarrow$ 128 $\rightarrow$ 1&64 $\rightarrow$ 128 $\rightarrow$ 1 &64 $\rightarrow$ 128 $\rightarrow$ 1  \\
        aggregation node embeddings &-  &- &- &mean&- \\
        pooling&max &mean &mean& - &max \\
        \hline
        node feature dim &64&64&64&64&64 \\
        tunable parameters &10991 & 11161 & 11127 &10396 &11535\\
        \hline
    \end{tabular}}
    \caption{Optimal hyperparameter values and number of tunable model parameters for TGN.}
    \label{tab:hyperparameters:tgn}
\end{table*}

\begin{table}
\centering
\begin{tabular}{|c|c|}
\hline
 Library/Software & Version \\
\hline
CUDA & cu121 \\
\texttt{torch}  & 2.4.1+cu121 \\
\texttt{torch\_cluster} & 1.6.3+pt24cu121 \\
\texttt{torch\_scatter} & 2.1.2+pt24cu121 \\
\texttt{torch\_sparse} & 0.6.18+pt24cu121 \\

\texttt{torch\_geometric} & 2.5.1\\
\texttt{pyg-lib} & 0.4.0+pt24cu121 \\
\texttt{pathpyG} & 0.2.0  \\
\hline
\end{tabular}
\medskip
\caption{Version of key dependencies used in the implementation of our experimental evaluation
\label{tab:libraries}}
\end{table}

More detailed information on library versions, parameter number of our models, as well as the computational resources used during training are included in Tables \Cref{tab:hyperparameters} --\Cref{tab:libraries} .

We ran our experiments on a container-based (Singularity) HPC environment with 4.512 CPU cores and 160 GPUs (122 x L40, 24 x L40s, 16 x H100).
For our experiments we used a total of less than 30 GPU hours for training and evaluation.
Upon acceptance of our manuscript, we will make the code of our experiments as well as a container description available to ensure the reproducibility of our results.

\section{Additional Results}
\label{sec:addresults}

In this section, we include additional results for the temporal graph classification experiment in synthetic temporal graphs for parameter pairs $\sigma_1, \sigma_2$, where we assign all temporal graphs generated for $\sigma_1$ to one class, while all temporal graphs generated for $\sigma_2$ are assigned to the other class.
In \Cref{fig:classification:stdev} we show the standard deviation of classification accuracies of our TGNN model, fitting the mean accuracies of our model reported in \Cref{fig:classificationRQ3} in \Cref{sec:results}.

\begin{figure}
    \centering
\includegraphics[width=0.45\linewidth]{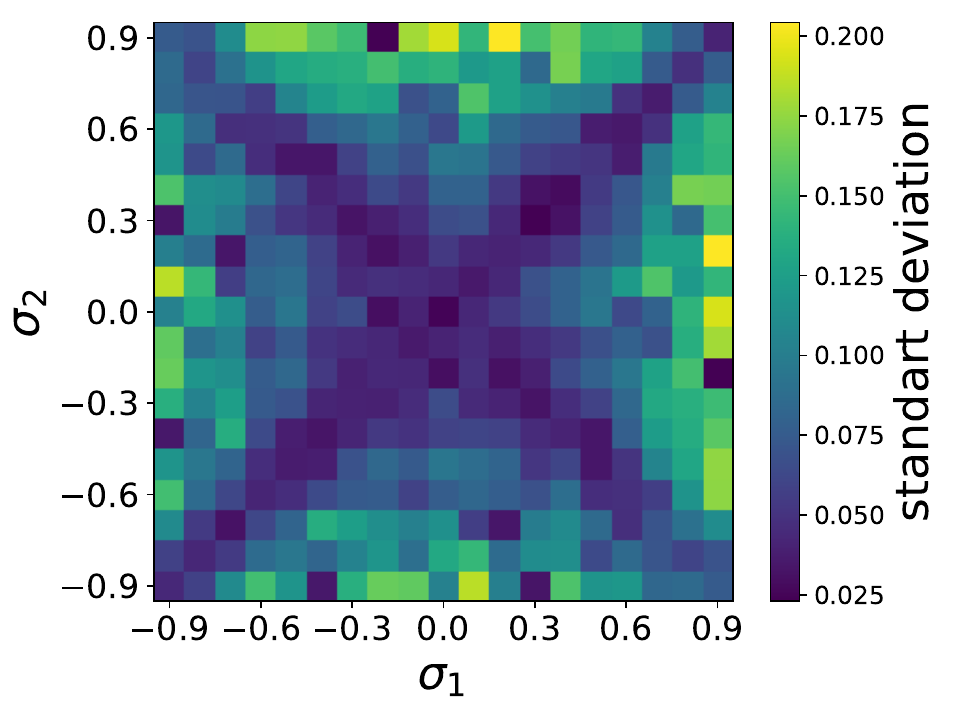}
\caption{Standard deviation of accuracies of our TGNN model for classification of temporal graphs with $\sigma_1$ vs. $\sigma_2$ for all pairs of $\sigma_1, \sigma_2$ (associated with \Cref{fig:classificationRQ3}).\label{fig:classification:stdev}}
\end{figure}

\Cref{fig:classification:gat} shows the mean accuracy of GAT applied to the time-concatenated static graph across 20 runs (mean standard deviation $0.048$).
We find that this model is not able to reliably classify temporal graphs for any combination of parameters $\sigma_1, \sigma_2$.

\begin{figure}
     \centering
   \begin{subfigure}[t]{0.45\textwidth}
   \centering
    \includegraphics[width=\linewidth]{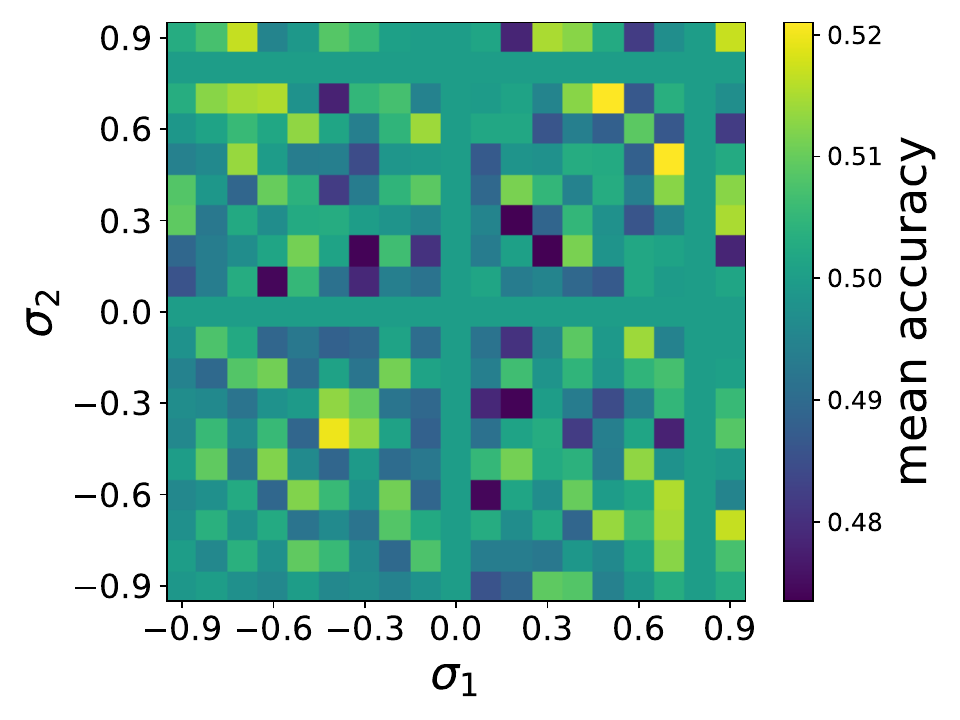}

    \label{fig:heatmap_gat}
    \end{subfigure}
    \begin{subfigure}[t]{0.45\textwidth}
    \centering

   \end{subfigure}
\caption{Mean accuracies of a GAT model on the time-concatenated static graph for classification of temporal graphs with $\sigma_1$ vs. $\sigma_2$ for all pairs of $\sigma_1, \sigma_2$. The mean standard deviation across all runs is $0.048$. \label{fig:classification:gat}}
\end{figure}

\section{Ablation Study}
\label{app:ablation}

\Cref{tab:ablation} shows results of an ablation study, where we selectively remove those aspects of our message passing architecture, which -- based on our theoretical analysis in \Cref{sec:messagepassing} -- we predict to be necessary to distinguish non-isomorphic temporal graphs.
We first remove node labels that indicate whether nodes in the augmented event graph represent nodes or timestamped edges in the original temporal graph (cf.\ $\nodeLabel$ in \Cref{def:aug-ev}).
We then remove the bidirectional message passing discussed in \Cref{sec:messagepassing}.
We finally ignore edge weights in the compressed event graph, as introduced in \Cref{app:compress}.
The results in \Cref{tab:ablation} demonstrate that the message passing architecture proposed in \Cref{sec:messagepassing}, which integrates all of those components, achieves the best performance (top-most row) compared to other approaches (other rows).

\begin{table}[]
    \centering
    \input{ablation_natcom}
    \medskip
    \caption{Results of ablation study for the classification task in the synthetic cluster model, using $\sigma_1=0$ and $\sigma_2=0.9$.}
    \label{tab:ablation}
\end{table}

\end{document}

%% file: fig/example_double_small.tex
\begin{tikzpicture}[xscale=\picWidth]

\node [node] (a) at (0,3) {};
\node [node] (b) at (0,2) {};
\node [node] (c) at (3,2) {};
\node [node] (d) at (1.5,1) {};
\node [node] (e) at (1.5,0) {};

\node at (a) {$a$};
\node at (b) {$b$};
\node at (c) {$c$};
\node at (d) {$d$};
\node at (e) {$e$};

\path[->] (a) edge[bend right=45] node[left] {$1$} (b);
\path[->] (a) edge[bend left=45] node[right] {$5$} (b);
\path[->] (b) edge[bend right=30] node[below left] {$2$} (d);
\path[->] (b) edge[bend left=30] node[above right] {$6$} (d);
\path[->] (c) edge[bend right=30] node[above left] {$3$} (d);
\path[->] (c) edge[bend left=30] node[below right] {$7$} (d);
\path[->] (d) edge[bend right=45] node[left] {$4$} (e);
\path[->] (d) edge[bend left=45] node[right] {$8$} (e);
\end{tikzpicture}

%% file: fig/augmented.tex
\begin{tikzpicture}[xscale=\picWidth,yscale=0.7]

\node [node,fill=black!15] (a) at (1,4) {};
\node [node,fill=black!15] (b) at (1,3) {};
\node [node,fill=black!15] (c) at (1,2) {};
\node [node,fill=black!15] (d) at (1,1) {};
\node [node,fill=black!15] (e) at (1,0) {};
\node [node,inner sep=6pt] (ab) at (2,3.5) {};
\node [node,inner sep=6pt] (bd) at (2,2.5) {};
\node [node,inner sep=6pt] (cd) at (2,1.5) {};
\node [node,inner sep=6pt] (de) at (2,0.5) {};
\node [node,inner sep=6pt] (ab2) at (0,3.5) {};
\node [node,inner sep=6pt] (bd2) at (0,2.5) {};
\node [node,inner sep=6pt] (cd2) at (0,1.5) {};
\node [node,inner sep=6pt] (de2) at (0,0.5) {};

\node at (a) {$a$};
\node at (b) {$b$};
\node at (c) {$c$};
\node at (d) {$d$};
\node at (e) {$e$};
\node at (ab) {\small $ab^1$};
\node at (bd) {\small $bd^2$};
\node at (cd) {\small $cd^3$};
\node at (de) {\small $de^4$};
\node at (ab2) {\small $ab^5$};
\node at (bd2) {\small $bd^6$};
\node at (cd2) {\small $cd^7$};
\node at (de2) {\small $de^8$};

\draw[->] (a) -- (ab);
\draw[->] (ab) -- (b);
\draw[->] (b) -- (bd);
\draw[->] (bd) -- (d);
\draw[->] (c) -- (cd);
\draw[->] (cd) -- (d);
\draw[->] (d) -- (de);
\draw[->] (de) -- (e);
\draw[->] (ab) -- (bd);
\path[->] (bd) edge[bend left=45] (de);
\draw[->] (cd) -- (de);
\draw[->] (a) -- (ab2);
\draw[->] (ab2) -- (b);
\draw[->] (b) -- (bd2);
\draw[->] (bd2) -- (d);
\draw[->] (c) -- (cd2);
\draw[->] (cd2) -- (d);
\draw[->] (d) -- (de2);
\draw[->] (de2) -- (e);
\draw[->] (ab2) -- (bd2);
\path[->] (bd2) edge[bend right=45] (de2);
\draw[->] (cd2) -- (de2);
\end{tikzpicture}

%% file: fig/example.tex
\begin{tikzpicture}[xscale=\picWidth,yscale=\picHeight]

\node [node] (a) at (0,3) {};
\node [node] (b) at (0,2) {};
\node [node] (c) at (2,3) {};
\node [node] (d) at (1,2) {};
\node [node] (e) at (1,1) {};

\node at (a) {$a$};
\node at (b) {$b$};
\node at (c) {$c$};
\node at (d) {$d$};
\node at (e) {$e$};

\draw[->] (a) -- node[left] {$1$} (b);
\draw[->] (b) -- node[above] {$2$} (d);
\draw[->] (c) -- node[below right] {$3$} (d);
\draw[->] (d) -- node[left] {$4$} (e);
\end{tikzpicture}

%% file: fig/example2.tex
\begin{tikzpicture}[xscale=\picWidth,yscale=\picHeight]

\node [node] (a) at (0,3) {};
\node [node] (b) at (0,2) {};
\node [node] (c) at (2,3) {};
\node [node] (d) at (1,2) {};
\node [node] (e) at (1,1) {};

\node at (a) {$a$};
\node at (b) {$b$};
\node at (c) {$c$};
\node at (d) {$d$};
\node at (e) {$e$};

\draw[->] (a) -- node[left] {$2$} (b);
\draw[->] (b) -- node[above] {$3$} (d);
\draw[->] (c) -- node[below right] {$4$} (d);
\draw[->] (d) -- node[left] {$5$} (e);
\end{tikzpicture}

%% file: fig/example3.tex
\begin{tikzpicture}[xscale=\picWidth,yscale=\picHeight]

\node [node] (a) at (0,3) {};
\node [node] (b) at (0,2) {};
\node [node] (c) at (2,3) {};
\node [node] (d) at (1,2) {};
\node [node] (e) at (1,1) {};

\node at (a) {$a$};
\node at (b) {$b$};
\node at (c) {$c$};
\node at (d) {$d$};
\node at (e) {$e$};

\draw[->] (a) -- node[left] {$1$} (b);
\draw[->] (b) -- node[above] {$3$} (d);
\draw[->] (c) -- node[below right] {$2$} (d);
\draw[->] (d) -- node[left] {$4$} (e);
\end{tikzpicture}

%% file: fig/example4.tex
\begin{tikzpicture}[xscale=\picWidth,yscale=\picHeight]

\node [node] (a) at (0,3) {};
\node [node] (b) at (0,2) {};
\node [node] (c) at (2,3) {};
\node [node] (d) at (1,2) {};
\node [node] (e) at (1,1) {};

\node at (a) {$a$};
\node at (b) {$b$};
\node at (c) {$c$};
\node at (d) {$d$};
\node at (e) {$e$};

\draw[->] (a) -- node[left] {$1$} (b);
\draw[->] (b) -- node[above] {$2$} (d);
\draw[->] (c) -- node[below right] {$3$} (d);
\draw[->] (d) -- node[left] {$1$} (e);
\end{tikzpicture}

%% file: fig/example5.tex
\begin{tikzpicture}[xscale=\picWidth,yscale=\picHeight]

\node [node] (a) at (0,3) {};
\node [node] (b) at (0,2) {};
\node [node] (c) at (2,3) {};
\node [node] (d) at (1,2) {};
\node [node] (e) at (1,1) {};

\node at (a) {$a$};
\node at (b) {$b$};
\node at (c) {$c$};
\node at (d) {$d$};
\node at (e) {$e$};

\path[->] (a) edge node[left] {$1$} (b);
\draw[->] (b) -- node[above] {$2$} (d);
\draw[->] (c) -- node[below right] {$3$} (d);
\path[->] (d) edge[bend right=45] node[left] {$4$} (e);
\path[->] (d) edge[bend left=45] node[right] {$5$} (e);
\end{tikzpicture}

%% file: fig/example_double.tex
\begin{tikzpicture}[xscale=\picWidth]

\node [node] (a) at (0,3) {};
\node [node] (b) at (0,2) {};
\node [node] (c) at (4,2) {};
\node [node] (d) at (2,1) {};
\node [node] (e) at (2,0) {};

\node at (a) {$a$};
\node at (b) {$b$};
\node at (c) {$c$};
\node at (d) {$d$};
\node at (e) {$e$};

\path[->] (a) edge[bend right=45] node[left] {$1$} (b);
\path[->] (a) edge[bend left=45] node[right] {$5$} (b);
\path[->] (b) edge[bend right=30] node[below left] {$2$} (d);
\path[->] (b) edge[bend left=30] node[above right] {$6$} (d);
\path[->] (c) edge[bend right=30] node[above left] {$3$} (d);
\path[->] (c) edge[bend left=30] node[below right] {$7$} (d);
\path[->] (d) edge[bend right=45] node[left] {$4$} (e);
\path[->] (d) edge[bend left=45] node[right] {$8$} (e);
\end{tikzpicture}

%% file: fig/augmented_compressed.tex
\begin{tikzpicture}[xscale=\picWidth,yscale=0.8]

\node [node,fill=black!15] (a) at (1,4) {};
\node [node,fill=black!15] (b) at (1,3) {};
\node [node,fill=black!15] (c) at (1,2) {};
\node [node,fill=black!15] (d) at (1,1) {};
\node [node,fill=black!15] (e) at (1,0) {};
\node [node,inner sep=6pt] (ab) at (2,3.5) {};
\node [node,inner sep=6pt] (bd) at (2,2.5) {};
\node [node,inner sep=6pt] (cd) at (2,1.5) {};
\node [node,inner sep=6pt] (de) at (2,0.5) {};

\node at (a) {$a$};
\node at (b) {$b$};
\node at (c) {$c$};
\node at (d) {$d$};
\node at (e) {$e$};
\node at (ab) {\small $ab^1$};
\node at (bd) {\small $bd^2$};
\node at (cd) {\small $cd^3$};
\node at (de) {\small $de^4$};

\draw[->] (a) -- (ab);
\draw[->] (ab) -- (b);
\draw[->] (b) -- (bd);
\draw[->] (bd) -- (d);
\draw[->] (c) -- (cd);
\draw[->] (cd) -- (d);
\draw[->] (d) -- (de);
\draw[->] (de) -- (e);
\path[->] (ab) edge[bend left=45] node[right] {$2$} (bd);
\path[->] (bd) edge[bend left=45] node[right] {$2$} (de);
\path[->] (cd) edge[bend left=45] node[right] {$2$} (de);
\end{tikzpicture}

%% file: fig/compressed_counterexample.tex
\begin{tikzpicture}[xscale=0.8,yscale=0.9]

\node [node,inner sep=6pt] (ab1) at (0,2) {};
\node [node,inner sep=6pt] (bc1a) at (0,1) {};
\node [node,inner sep=6pt] (bc1b) at (2,1) {};
\node [node,inner sep=6pt] (cd1) at (1,0) {};

\node [node,inner sep=6pt] (ab2) at (4,2) {};
\node [node,inner sep=6pt] (bc2a) at (4,1) {};
\node [node,inner sep=6pt] (bc2b) at (6,1) {};
\node [node,inner sep=6pt] (cd2) at (5,0) {};

\node [node,inner sep=6pt] (ab3) at (8,2) {};
\node [node,inner sep=6pt] (bc3a) at (8,1) {};
\node [node,inner sep=6pt] (bc3b) at (10,1) {};
\node [node,inner sep=6pt] (cd3) at (9,0) {};

\node [node,inner sep=6pt] (ab4) at (12,2) {};
\node [node,inner sep=6pt] (bc4a) at (12,1) {};
\node [node,inner sep=6pt] (bc4b) at (14,1) {};
\node [node,inner sep=6pt] (cd4) at (13,0) {};

\node at (ab1) {\small $ab^1$};
\node at (bc1a) {\small $bc^3$};
\node at (bc1b) {\small $bc^4$};
\node at (cd1) {\small $cd^6$};

\node at (ab2) {\small $ab^{11}$};
\node at (bc2a) {\small $bc^{13}$};
\node at (bc2b) {\small $bc^{14}$};
\node at (cd2) {\small $cd^{16}$};

\node at (ab3) {\small $ab^{21}$};
\node at (bc3a) {\small $bc^{22}$};
\node at (bc3b) {\small $bc^{21}$};
\node at (cd3) {\small $cd^{23}$};

\node at (ab4) {\small $ab^{31}$};
\node at (bc4a) {\small $bc^{32}$};
\node at (bc4b) {\small $bc^{31}$};
\node at (cd4) {\small $cd^{33}$};

\draw[->] (ab1) -- (bc1a);
\draw[->] (bc1a) -- (cd1);
\draw[->] (bc1b) -- (cd1);
\draw[->] (ab2) -- (bc2a);
\draw[->] (bc2a) -- (cd2);
\draw[->] (bc2b) -- (cd2);
\draw[->] (ab3) -- (bc3a);
\draw[->] (bc3a) -- (cd3);
\draw[->] (bc3b) -- (cd3);
\draw[->] (ab4) -- (bc4a);
\draw[->] (bc4a) -- (cd4);
\draw[->] (bc4b) -- (cd4);
\end{tikzpicture}

%% file: fig/compressed_counterexample_2.tex
\begin{tikzpicture}[xscale=0.8,yscale=0.9]

\node [node,inner sep=6pt] (ab1) at (0,2) {};
\node [node,inner sep=6pt] (bc1a) at (0,1) {};
\node [node,inner sep=6pt] (bc1b) at (2,1) {};
\node [node,inner sep=6pt] (cd1) at (1,0) {};

\node [node,inner sep=6pt] (ab2) at (4,2) {};
\node [node,inner sep=6pt] (bc2a) at (4,1) {};
\node [node,inner sep=6pt] (bc2b) at (6,1) {};
\node [node,inner sep=6pt] (cd2) at (5,0) {};

\node [node,inner sep=6pt] (ab3) at (8,2) {};
\node [node,inner sep=6pt] (bc3a) at (8,1) {};
\node [node,inner sep=6pt] (bc3b) at (10,1) {};
\node [node,inner sep=6pt] (cd3) at (9,0) {};

\node [node,inner sep=6pt] (ab4) at (12,2) {};
\node [node,inner sep=6pt] (bc4a) at (12,1) {};
\node [node,inner sep=6pt] (bc4b) at (14,1) {};
\node [node,inner sep=6pt] (cd4) at (13,0) {};

\node at (ab1) {\small $ab^1$};
\node at (bc1a) {\small $bc^3$};
\node at (bc1b) {\small $bc^4$};
\node at (cd1) {\small $cd^6$};

\node at (ab2) {\small $ab^{11}$};
\node at (bc2a) {\small $bc^{12}$};
\node at (bc2b) {\small $bc^{11}$};
\node at (cd2) {\small $cd^{13}$};

\node at (ab3) {\small $ab^{21}$};
\node at (bc3a) {\small $bc^{22}$};
\node at (bc3b) {\small $bc^{21}$};
\node at (cd3) {\small $cd^{23}$};

\node at (ab4) {\small $ab^{31}$};
\node at (bc4a) {\small $bc^{32}$};
\node at (bc4b) {\small $bc^{31}$};
\node at (cd4) {\small $cd^{33}$};

\draw[->] (ab1) -- (bc1a);
\draw[->] (bc1a) -- (cd1);
\draw[->] (bc1b) -- (cd1);
\draw[->] (ab2) -- (bc2a);
\draw[->] (bc2a) -- (cd2);
\draw[->] (bc2b) -- (cd2);
\draw[->] (ab3) -- (bc3a);
\draw[->] (bc3a) -- (cd3);
\draw[->] (bc3b) -- (cd3);
\draw[->] (ab4) -- (bc4a);
\draw[->] (bc4a) -- (cd4);
\draw[->] (bc4b) -- (cd4);
\end{tikzpicture}

%% file: fig/compressed_counterexample_3.tex
\begin{tikzpicture}[xscale=0.8,yscale=0.9]

\node [node,inner sep=6pt] (ab1) at (0,2) {};
\node [node,inner sep=6pt] (bc1a) at (0,1) {};
\node [node,inner sep=6pt] (bc1b) at (2,1) {};
\node [node,inner sep=6pt] (cd1) at (1,0) {};

\node [node,inner sep=6pt] (ab2) at (4,2) {};
\node [node,inner sep=6pt] (bc2a) at (4,1) {};
\node [node,inner sep=6pt] (bc2b) at (6,1) {};
\node [node,inner sep=6pt] (cd2) at (5,0) {};

\node at (ab1) {\small $ab^1$};
\node at (bc1a) {\small $bc^3$};
\node at (bc1b) {\small $bc^4$};
\node at (cd1) {\small $cd^6$};

\node at (ab2) {\small $ab^{21}$};
\node at (bc2a) {\small $bc^{22}$};
\node at (bc2b) {\small $bc^{21}$};
\node at (cd2) {\small $cd^{23}$};

\path[->] (ab1) edge node[left] {$2$} (bc1a);
\draw[->] (bc1a) edge node[below left] {$2$} (cd1);
\draw[->] (bc1b) edge node[below right] {$2$} (cd1);
\draw[->] (ab2) edge node[left] {$2$} (bc2a);
\draw[->] (bc2a) edge node[below left] {$2$} (cd2);
\draw[->] (bc2b) edge node[below right] {$2$} (cd2);
\end{tikzpicture}

%% file: fig/compressed_counterexample_4.tex
\begin{tikzpicture}[xscale=0.8,yscale=0.9]

\node [node,inner sep=6pt] (ab1) at (0,2) {};
\node [node,inner sep=6pt] (bc1a) at (0,1) {};
\node [node,inner sep=6pt] (bc1b) at (2,1) {};
\node [node,inner sep=6pt] (cd1) at (1,0) {};

\node [node,inner sep=6pt] (ab2) at (4,2) {};
\node [node,inner sep=6pt] (bc2a) at (4,1) {};
\node [node,inner sep=6pt] (bc2b) at (6,1) {};
\node [node,inner sep=6pt] (cd2) at (5,0) {};

\node at (ab1) {\small $ab^1$};
\node at (bc1a) {\small $bc^3$};
\node at (bc1b) {\small $bc^4$};
\node at (cd1) {\small $cd^6$};

\node at (ab2) {\small $ab^{21}$};
\node at (bc2a) {\small $bc^{22}$};
\node at (bc2b) {\small $bc^{21}$};
\node at (cd2) {\small $cd^{23}$};

\path[->] (ab1) edge node[left] {$1$} (bc1a);
\draw[->] (bc1a) edge node[below left] {$1$} (cd1);
\draw[->] (bc1b) edge node[below right] {$1$} (cd1);
\draw[->] (ab2) edge node[left] {$3$} (bc2a);
\draw[->] (bc2a) edge node[below left] {$3$} (cd2);
\draw[->] (bc2b) edge node[below right] {$3$} (cd2);
\end{tikzpicture}

%% file: ablation_natcom.tex
\begin{tabular}{lc}
\toprule
Experiment & Accuracy \\
\midrule
Compressed augmented event graph & 0.95 $\pm$ 0.03 \\
\midrule 
No node labels & 0.66 $\pm$ 0.20 \\
Message passing only in one direction & 0.63 $\pm$ 0.18 \\
No edge weights & 0.81 $\pm$ 0.22 \\
\bottomrule
\end{tabular}